\pdfoutput=1 
\documentclass{article}

\usepackage[utf8]{inputenc} %
\usepackage[T1]{fontenc}    %
\usepackage[colorlinks,citecolor=blue!70!black,linkcolor=blue!70!black,
urlcolor=blue!70!black]{hyperref}       %
\usepackage{url}            %
\usepackage{booktabs}       %
\usepackage{amsfonts}       %
\usepackage{nicefrac}       %
\usepackage{microtype}      %
\usepackage{xcolor}         %
\usepackage{wrapfig}
\usepackage{sidecap}
\usepackage{lipsum}
\usepackage{fullpage}

\usepackage{graphicx}
\usepackage{subfigure}
\usepackage{enumitem}
\usepackage{multirow}
\usepackage{floatrow}
\newfloatcommand{capbtabbox}{table}[][\FBwidth]

\usepackage{xspace}

\usepackage{amsmath}
\usepackage{amssymb}
\usepackage{mathtools}
\usepackage{amsthm}

\usepackage{amsfonts,dsfont}
\usepackage{tikz}
\usepackage{comment} 

\usepackage{thmtools,thm-restate}

\usepackage{std_definitions}

\allowdisplaybreaks
\usepackage{natbib}

\newcommand{\fhat}{\hat{f}}

\newcommand{\Reg}{\textrm{Reg}}

\newcommand{\algname}{{\tt LORIL}}
\newcommand{\bonus}{b_t}

\newcommand{\protname}{{\tt LHI}}

\newcommand{\TV}[1]{\nbr{#1}_{\textrm{TV}}}

\usepackage{algorithm}
\usepackage{algorithmicx}
\usepackage[noend]{algpseudocode}

\usepackage[capitalize,noabbrev]{cleveref}

\usepackage{caption}

\title{Provable Interactive Learning with Hindsight Instruction Feedback}

\author{
  Dipendra Misra$^{1,\star}$ \qquad Aldo Pacchiano$^{2, 3,\star}$ \qquad Robert E. Schapire$^{1, \star}$  \\
  \vspace{2mm} \\
  Microsoft Research New York$^1$ \quad Boston University$^2$\quad Broad Institute of MIT and Harvard$^3$\\[.1cm]
  ($\star$ equal contribution)\\[.1cm]
\texttt{dimisra@microsoft.com, pacchian@bu.edu, schapire@microsoft.com} 
}
\date{}

\newcounter{protocol}
\makeatletter
\newenvironment{protocol}[1][htb]{%
  \let\c@algorithm\c@protocol
  \renewcommand{\ALG@name}{Protocol}%
  \begin{algorithm}[#1]%
  }{\end{algorithm}
}
\makeatother

\begin{document}
\maketitle

\begin{abstract}
We study interactive learning in a setting where the agent has to generate a response (e.g., an action or trajectory) given a context and an instruction. In contrast, to typical approaches that train the system using reward or expert supervision on response, we study \emph{learning with hindsight labeling} where a teacher provides an instruction that is most suitable for the agent's generated response. This hindsight labeling of instruction is often easier to provide than providing expert supervision of the optimal response which may require expert knowledge or can be impractical to elicit. We initiate the theoretical analysis of \emph{interactive learning with hindsight labeling}. We first provide a lower bound showing that in general, the regret of any algorithm must scale with the size of the agent's response space. Next we study a specialized setting where the underlying instruction-response distribution can be decomposed as a low-rank matrix. We introduce an algorithm called $\algname$ for this setting and show that its regret scales
with $\sqrt{T}$ and depends on the \emph{intrinsic rank} but does not depend on the size of the agent's response space. We provide experiments in two domains showing that $\algname$ outperforms baselines even when the low-rank assumption is violated.%
\end{abstract}

\section{Introduction}
\label{sec:intro}

Success of a machine learning approach is intimately tied to the ease of getting training data. For example, language models~\citep{brown2020language,achiam2023gpt}, which are one of the most successful applications of machine learning, are trained on an abundance of language data which is both easy to elicit from non-expert users and is available offline. In contrast, consider the task of a robot following instructions specified by a human user~\citep{misra2016tell,blukis2019,myers2023goal}. It is expensive to collect ground truth robot trajectories making standard imitation learning (IL) approaches~\citep{pomerleau1991efficient} expensive to apply, whereas reinforcement learning (RL)~\citep{sutton2018reinforcement} approaches suffer from high sample complexity. This makes IL and RL-- the two most common ways of training agents, expensive in practice. 
Motivated by the limitations of IL and RL, a line of work has proposed using \emph{hindsight labeling}, where the agent (robot in our example) generates a response (trajectory) given an instruction, and a teacher instead of providing expensive ground truth response, provides the instruction that is suitable for the agent's response~\citep{fried2018speaker,nguyen2021interactive}. This reverses the labeling problem to an easier labeling problem, since instructions are typically in a format such as natural language, which can be inexpensively elicited from non-expert users in contrast to robot trajectories. While this approach has been applied empirically, a theoretical understanding remains absent. In this work, we initiate the theoretical understanding of interactive learning from hindsight instruction.

\begin{figure*}
\includegraphics[scale=0.29]{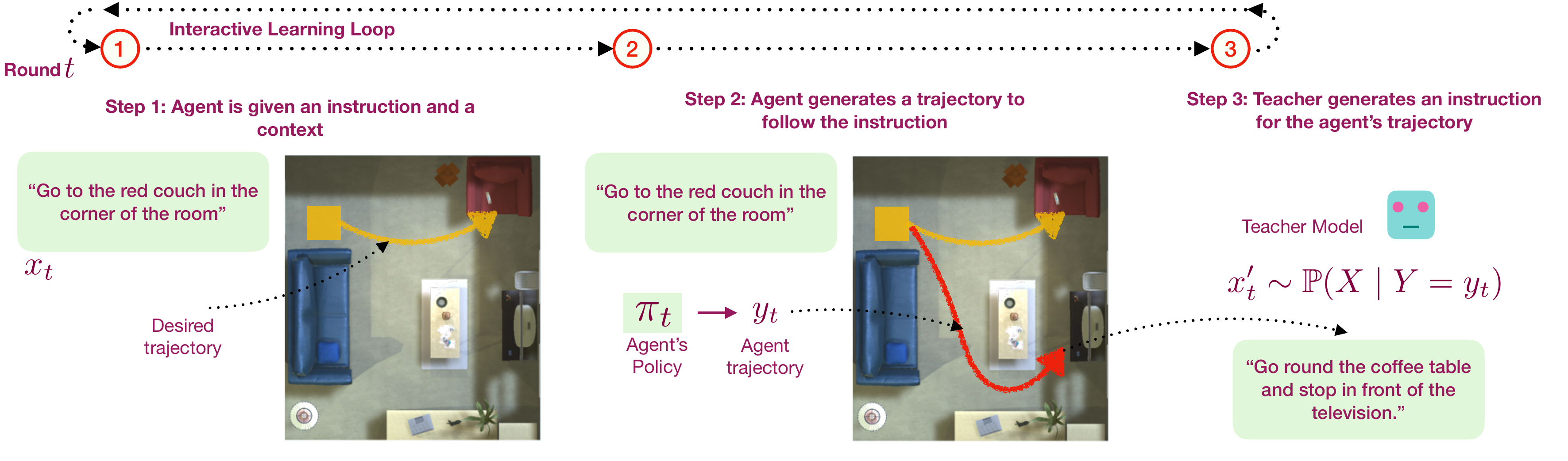}
\caption{Shows sketch of our interactive \textbf{L}earning from \textbf{H}indsight \textbf{I}nstruction (\protname) setting. The agent interacts with the world iteratively. In each round (or time step), the agent is given an instruction $x_t$ and a context $s_t$. In our case, the context $s_t$ is the house layout. In response, the agent generates a trajectory (response) $y_t$ which is then labeled by a teacher model with an instruction $x'_t$ (\emph{hindsight instruction}). The agent \emph{never} receives any expert response or rewards.}
\label{fig:flowchart}
\end{figure*}

We consider the learning setup illustrated in~\pref{fig:flowchart}. In this setup, a teacher is teaching an agent to navigate in a virtual home environment. In each round, the world gives an instruction and a context to the agent. The instruction in this case is expressed in natural language. The context is an image that provides information about the environment such as the position, color, and sizes of different objects. The goal of the agent is to generate a trajectory that follows the given instruction. In the beginning, the agent lacks any language understanding and, therefore, cannot generate correct trajectories. We assume access to a teacher that can provide an instruction that best describes the agent's trajectory. This type of feedback can be viewed as a \emph{hindsight instruction}, as it was the correct instruction in hindsight for the trajectory generated by the agent. In each round, the agent generates a trajectory and receives a hindsight instruction from the teacher. This allows the agent to learn a mapping from the instruction space to the trajectory space, which helps improve the agent's policy. We call this learning approach as \textbf{L}earning from \textbf{H}indsight \textbf{I}nstruction ($\protname$).

There are several different approaches for training a decision-making agent. One of the most commonly used approaches is \emph{imitation learning} (IL) where a teacher provides access to expert demonstrations allowing the agent to learn the right behavior~\citep{ross2011reduction}. For the example in~\pref{fig:flowchart}, this will require the teacher to be able to understand the agent's action space and dynamics. This often requires domain expertise and can only be provided by expert teachers and may require specialized tools.\footnote{One such commonly used approach is motion capture where a human can record behavior that can be transferred to a humanoid agent but this requires specialized tools.} In contrast, a non-expert user can easily provide an instruction for the red trajectory in~\pref{fig:flowchart}. 

Reinforcement learning (RL) is another widely used approach for training agents that overcomes the expense of collecting expert demonstrations by directly optimizing a reward function that is more user-friendly to provide~\citep{sutton2018reinforcement}. However, RL approaches are less sample efficient than IL approaches making them less suited for real-world settings.%
In contrast, learning from hindsight instruction uses instruction feedback which is user-friendly and more natural for humans to provide. Further, instruction feedback contains significantly richer information than scalar rewards which can help in reducing the sample complexity compared to RL~\citep{nguyen2021interactive}.

Because of its promise, learning from hindsight labeling has been explored in various applications~\citep{andrychowicz2017hindsight,fried2018speaker,nguyen2021interactive}. However, a principled understanding of this setting remains absent despite these empirical results. In particular, we focus on an interactive learning setting where a teacher trains the agent using hindsight instructions. For these settings, the natural evaluation metric is regret which penalizes the agent for failing to follow a given instruction. A key challenge in designing algorithms for this setting is that the agent has to both \emph{exploit} to follow the given instruction, but also \emph{explore} to improve its understanding capabilities more generally. In this work, we initiate the theoretical study of this setting. We first present a formal interactive learning setting and define a notion of regret. Motivated by natural settings where the teacher is a human user, we assume access to a black box teacher which can generate a sampled instruction given the agent's response but where the agent \emph{does not have access} to the teacher's probability values. The agent is evaluated using a \emph{hidden reward} given by the probability of the teacher labeling the agent's response by the original given instruction.

We first prove a lower bound for this setting showing that in the worst case, the regret bounds for any algorithm will scale polynomially with the size of the agent's response space. In many applications such as our robot navigation example, the agent's response is a trajectory, leading to an exponentially large response space. However, in practice using function approximations and featurization enables generalization in infinitely large spaces. Motivated by this we introduce a low-rank setting where where the agent has access to a feature representation of the response and context and the teacher's distribution admits a low-rank decomposition in this feature space. %
We introduce an algorithm $\algname$ for this setting and derive regret bounds that scale as $\sqrt{T}$ with the horizon $T$. Importantly, the regret \emph{does not depend} upon the size of the agent's response or the size of instruction or context space, and instead depends on the rank of the teacher's distribution, which can be significantly smaller in practice.

We evaluate $\algname$ on two tasks. In the first setting, we use a synthetic task where low-rank assumption holds and show that $\algname$ achieves lower regret compared to baselines. In the second setting, we apply $\algname$ to a setting with natural language instruction and images and show that insights from $\algname$ help \emph{even when the low-rank assumption does not hold}.

We include a discussion on the related literature after presenting our results in~\pref{sec:related-work}. The code for all experiments in the paper can be found at~\url{https://github.com/microsoft/Intrepid}.

\section{Preliminary and Overview}
\label{sec:prelim}

We first introduce the mathematical notations before providing an overview of our setup.

\paragraph{Notation.} For a given $N \in \NN$, we define $[N] = \{1, 2, \cdots, N\}$. For a given countable set $\Ucal$, we denote the set of all distributions over $\Ucal$ by $\Delta(\Ucal)$. For a given positive-definite matrix $A \in \RR^{d \times d}$, we define the induced norm of a vector $v \in \RR^d$ as $\nbr{v}_{A} = \sqrt{v^\top A v}$. %

\paragraph{Interactive Learning from Hindsight Instruction.} We define the space of contexts as $\Scal$, the space of instructions as $\Xcal$ and the space of all possible agent response as $\Ycal$. We assume $\Scal$, $\Xcal$, and $\Ycal$ to be finite for our analysis but allow them to be arbitrarily large. The finiteness is only a mild assumption in practice, as it still allows us to handle the most common data types. For example, if $\Xcal$ denotes the space of all $m\times n$ RGB images with each pixel taking values in $\{0, 1, \cdots, 255\}$, then $|\Xcal| = 256^{3mn}$ which is an exponentially large but finite value.

\begin{protocol}[t]
\caption{Shows the protocol for our setting: Learning from Hindsight Instruction (\protname). The line in \textcolor{blue}{blue} needs to be implemented by an algorithm implementing the protocol.}
\label{proto:protocol}
\begin{algorithmic}[1]
\For{$t=1, 2, \cdots, T$}
\State World presents $s_t, x_t$ possibly adversarially
\State \textcolor{blue}{Agent generates a response $y_t \in \Ycal$}
\State Teacher describes the response $x'_t \sim \PP(X \mid y_t, s_t)$
\State Evaluate using a hidden reward $r_t = \PP(x_t \mid y_t, s_t)$
\EndFor
\State Return $\sum_{t=1}^T r_t$
\end{algorithmic}
\end{protocol}

The agent interacts with the world repeatedly over $T$ rounds.~\pref{proto:protocol} shows our learning framework. In the $t^{th}$ round, the world presents a context $s_t$ and an instruction $x_t$ sampled according to a fixed distribution $D_t(\cdot, \cdot \mid x_1, y_1, s_1, \cdots, x_{t-1}, s_{t-1}, y_{t-1})$ that can depend on the past history. Given the instruction and the context, the agent generates a response $y_t \in \Ycal$. Ideally, we want the agent to generate a response that fulfills the intent of the instruction. After generating the response, the agent receives a \emph{hindsight instruction} $x'_t$ sampled from a fixed conditional distribution model $\PP(X \mid Y=y_t, S = s_t)$. This conditional distribution models a \emph{teacher} that provides an instruction that is most appropriate for the agent response. In a typical setting, this teacher will be modeled using a human-in-the-loop.

The agent \emph{does not} have access to $\PP(X \mid Y=y_t, S=s_t)$ but can observe a sample from this distribution by generating a response and asking the teacher to label it with an instruction. This is because in a human-in-the-loop setting, we don't have access to the human teacher's distribution. %

The teacher model $\PP(X \mid Y, S = s_t)$ is in practice highly stochastic since there can be many possible ways to describe instructions for a given response. Further, the space of all possible responses and instructions can be impractically large, necessitating the use of function approximation.

\paragraph{Computing Regret.} Given a state $s_t$ and context $x_t$, the ideal response should maximize the probability of the teacher labeling the response with the right instruction $x_t$, i.e., the agent should play $y = \arg\max_{y \in \Ycal} \PP(x_t \mid y, s_t)$. We can, therefore, view $\PP(x_t \mid y, s_t)$ as a \emph{latent reward} for generating response $y$. This leads to a natural notion of regret given by:
\begin{equation}\label{eqn:regret}
    \Reg(T) = \sum_{t=1}^T \rbr{\max_{y \in \Ycal} \PP(x_t \mid s_t, y) - \PP(x_t \mid s_t, y_t)},
\end{equation}
where $s_t, x_t$ are the context and instruction in round $t$ and $y_t$ is the response generated by the agent. 

There can be alternative ways to define regret in~\pref{eqn:regret}.
Log-probabilities $\log \PP(X \mid S, Y)$ may appear more natural to use instead of probabilities, however, the former is unbounded which makes it ill-suited for defining reward. For example, if in the first round, the agent generates a response $y_1$ for which $\PP(x_1 \mid y_1, s_1) = 0$, then the agent's regret is unbounded irrespective of the agent's performance in later rounds. Another choice for reward is the likelihood of the response $\PP(y \mid x_t, s_t)$. However, this requires assuming a prior distribution over $y$ which can be hard to realize.

\section{Lower Bound in the General Case}
\label{sec:lower-bound}

We first prove that it is impossible to design an algorithm for~\pref{proto:protocol} with a regret bound that doesn't scale polynomially in the size of plausible responses $|\Ycal|$. 

We introduce the concept of `stochastic worlds' to prove our lower bound. A stochastic world $W$ consists of a set of instructions, contexts marginal distribution $\PP_W(X, S)$ and a conditional distribution of instructions given responses and contexts $\PP_W(X|Y, S)$. 

When at time $t$ an agent $\mathbb{A}$ interacts via \pref{proto:protocol} with a stochastic world $W$, the world produces an instruction $x_t$ and context $s_t$ sampled from a \emph{time-independent} distribution $\PP_W(X, S)$.  We use the notation $\mathbb{P}_{W, \mathbb{A}}$ and $\mathbb{E}_{W, \mathbb{A}}$ to denote the measure and expectations over trajectories $(s_1, x_1, y_1, x_1', \cdots, s_T, x_T, y_T, x_T')$ resulting from the interaction between $\mathbb{A}$ and world $W$. We show that for any $K \in \mathbb{N}$, and any algorithm $\mathbb{A}$, there is a stochastic world where the regret of algorithm $\mathbb{A}$ satisfies $\Reg(T) \geq \Omega(\sqrt{KT})$ when $\mathbb{A}$ interacts with $W$ through \pref{proto:protocol}.

To prove our main result we exhibit a family of stochastic worlds $\{W_i\}$, such that world $W_i$ is defined by context space $S = \{ s_o\}$ instruction space $\Xcal = \{A, B\}$, response set $\Ycal = [K]$ marginal $\PP_{W_i}(X, S ) = \mathrm{Uniform}((A,s_o),(B,s_o))$ for all $i \in [K]$, and conditional
\begin{equation*}
    \PP_{W_i}(X | y_j) =   1/2 + \sqrt{K/T} \cdot \mathbf{1}(j = i)\cdot(1-2\cdot\mathbf{1}(X = B)) .
\end{equation*} The context distribution is a delta mass around context $s_o$. In world $W_i$ the optimal response for instruction $X = A$ equals $y_{i}$, and the optimal response for instruction $X = B$ is any $y_j$ for  $ j \neq i $. Any suboptimal decision, regardless of the instruction incurs in regret of order $\sqrt{K/T}$. Our main result is the following.

\begin{restatable}{theorem}{lemmalowerbound}\label{lemma::lower_bound}
Let $T \geq 256\log(2e)$ and $K \geq 8e$. For any algorithm, there is at least one stochastic world $W_{\hat{i}}$ such that $\Reg(T)  \geq \frac{\sqrt{KT}}{8}$ such that with probability at least $1/4e$. 
\end{restatable}

The proof can be found in Appendix~\ref{section::lower_bound_appendix}. Lemma~\ref{lemma::lower_bound} implies the expected regret lower bound,

\begin{restatable}{corollary}{expectedregretlowerbound}\label{corollary::expected_regret_lowerbound} If the conditions of Lemma~\ref{lemma::lower_bound} hold then for any algorithm there exists at least one stochastic world $W_{\hat{i}}$ such that $\overline{\Reg}(T) \geq\Omega(\sqrt{KT})$. Where,
\begin{equation*}
\overline{\Reg}(T) = \mathbb{E}_{W_{\hat{i}}, \mathbb{A}}\left[ \sum_{t=1}^T \max_{y \in Y} \PP(  x_t |y, s_o)  - \PP(x_t | y_t, s_o )\right].
\end{equation*}
\end{restatable}
The proof of this result can also be found in Appendix~\ref{section::lower_bound_appendix}. Theorem~\ref{lemma::lower_bound} shows that for tractable hindsight learning it is necessary to impose structural assumptions on the conditional probabilities $\PP(X|Y, S)$. We explore one such assumption in the next sections.

\section{Provable Learning in Low-Rank Setting}
\label{sec:upper-bound}

The analysis in~\pref{sec:lower-bound} shows that the regret scales as $\Omega(\sqrt{|\Ycal|})$ which makes this an intractable setting when $\Ycal$ is extremely large. For settings with typically large input or output spaces, it is natural in practice to use function approximation. For example, a trajectory can be encoded using a neural network to a representation that contains the relevant information. In statistical learning theory, significant progress has been made in the study of learning with function approximation~\citep{misra2020kinematic,sekhari2021agnostic,foster2021statistical}. In particular, problems with low-rank structures~\citep{agarwal2020flambe,jin2020provably} have received significant attention due to their abilities to model commonly occurring settings and the success of corresponding algorithms in real-world problems even where the low-rank assumption is violated~\citep{henaff2022exploration}. Motivated by this, we introduce and study a setting where the teacher model $\PP(X \mid Y, S)$ admits a low-rank decomposition.

\paragraph{Low-Rank Teacher Model.} We consider a specialization of our general setup where the teacher model follows a low-rank decomposition. Formally, we assume that there exists $f^\star: \Xcal \rightarrow \RR^d$ and $g^\star: \Ycal \rightarrow \RR^d$ such that
\begin{equation*}
\forall s\in \Scal, x \in \Xcal, y \in \Ycal, \qquad    \PP(x \mid y, s) = f^\star(x)^\top g^\star(y, s),
\end{equation*}
where $d$ is the \emph{intrinsic dimension} of the problem which is much smaller than the size of $\Scal, \Xcal$, and $\Ycal$ which can all be infinitely large. We assume that the agent has knowledge of $g^\star$ but does not know $f^\star$.%

We assume access to a model class $\Fcal$ to learn $f^\star$. Our goal is to get regret guarantees that do not scale with $|\Xcal|, |\Ycal|, |\Scal|$ and instead only depend on the intrinsic dimension $d$ of the problem and the statistical complexity of $\Fcal$.

\paragraph{$\algname$ Algorithm.} We present the ``Learning in \textbf{LO}w-\textbf{R}ank models from \textbf{I}nstruction \textbf{L}abels" algorithm ($\algname$): for low-rank teacher models in~\pref{alg:known_gstar}. The algorithm assumes access to the embedding function $g^\star$ for encoding the agent's response. In practice, such a function can be available either using a pre-trained representation model or by using a self-supervised learning objective such as autoencoding. We discuss some implementation choices later in the experiment section. %

$\algname$ implements~\pref{proto:protocol}. In the $t^{th}$ round, the algorithm first computes a maximum likelihood estimation $\fhat_t$ of $f^\star$ using the historical data (\pref{line:known-g-mle-estimate}). We use this to define a policy $\pi_t$ to generate a response $y_t$. $\algname$ is based on the principle of optimism under uncertainty. As per this principle, we first compute an appropriate uncertainty measure $\bonus(y)$ for a response $y \in \Ycal$ such that we know with high probability that the true value of a response, i.e., $f^\star(x_t)^\top g^\star(y,s_t)$ lies in $\sbr{\fhat_t(x_t)^\top g^\star(y,s_t) - \bonus(y, s_t), \fhat_t(x_t)^\top g^\star(y,s_t) + \bonus(y, s_t)}$ with high probability. As $\fhat_t(x_t)^\top g^\star(y,s_t)$ is the current estimate of the value of a response $y$ in the $t^{th}$ round, we can view $b(y, s_t)$ as defining a confidence interval for a given response $y$ and context $s_t$. Second, we take the action that has the maximum possible value in the confidence interval, namely:
\begin{equation}\label{eqn:elliptic-bonus}
    y_t = \arg\max_{y \in \Ycal} \quad \biggl(\underbrace{\fhat_t(x_t)^\top g^\star(y,s_t)}_\text{estimated model value} + \underbrace{b_t(y, s_t)}_\text{bonus}\biggr).
\end{equation}
For low-rank models, we will show that $b_t(y, s_t)$ can be expressed as $\Ocal(\nbr{g^\star(y,s_t)}_{\widehat{\Sigma}^{-1}_t})$ where $\widehat{\Sigma}_t = \lambda_t \II + \sum_{l=1}^{t-1} g^\star(y_l,s_l) g^\star(y_l, ,s_l)^\top$ is a positive definite matrix capturing information about historical data. This can be viewed as a positive definite matrix $\lambda_t \II$ were $\lambda_t > 0$ and a sum of rank one positive semi-definite matrices $g^\star(y_l,s_l) g^\star(y_l,s_l)^\top$  which form a covariance matrix $\sum_{l=1}^{t-1} g^\star(y_l,s_l) g^\star(y_l, s_l)^\top$. The quantity $b_t(y, s_t)$ can be viewed as a bonus in~\pref{eqn:elliptic-bonus} and is known as \emph{elliptic bonus} in the literature and is a frequently appearing quantity in the study of linear models~\citep{abbasi2011improved} and low-rank models~\citep{agarwal2020flambe}.\footnote{The word elliptic comes from the fact that for a positive definite matrix $\Sigma$, the set $\{y \mid y^\top \Sigma^{-1} y \le 1\}$ denotes an ellipsoid centered at 0.}

The agent computes the optimistic response $y_t = \pi_t(x_t)$ (\pref{line:alg-generate-response}) and plays it. In response, the teacher provides a description $x'_t$ (\pref{line:alg-teacher-response}) which is added along with the agent response to the historical data.

Note that the agent never has direct access to $f^\star$ or the true model $\PP(X \mid Y, S)$, but only has access through feedback generated by the teacher model and through its knowledge of $g^\star$ and $\Fcal$. Further, the agent is also unaware of the true horizon length $T$ which is often unknown in practice.

\paragraph{Computational Efficiency.} The computation of the covariance matrix can be performed easily as can the computation of bonus $b_t$ for a given response. The inverse of the covariance matrix can be computed efficiently in a numerically stable way using the Sherman–Morrison formula~\citep{sherman1949adjustment}. The two main computationally expensive steps in $\algname$ are maximum-likelihood estimation and solving the optimization in~\pref{line:policy_definition}-\ref{line:alg-generate-response}.
\begin{algorithm}[!t]
    \caption{\textbf{$\algname$($g^\star$, $\Fcal$)}: Learning in \textbf{LO}w-\textbf{R}ank models from \textbf{I}nstruction \textbf{L}abels}\label{alg:known_gstar}
    \begin{algorithmic}[1]
        \Require Response embedding function $g^\star: \Scal \times \Ycal \rightarrow \RR^d$
        \Require Model class $\Fcal = \{f: \Xcal \rightarrow \RR^d\}$
        \State Define $\lambda_t = \frac{1}{t}$.
        \For{ $t=1, 2, \cdots, T$}
            \State Compute MLE estimator $\widehat{f}_t$ using $\{x_\ell', y_\ell, s_\ell\}_{\ell=1}^{t-1}$.\label{line:known-g-mle-estimate}
            \begin{equation*}
                \fhat_t = \arg\max_{f \in \Fcal} \sum_{\ell=1}^{t-1} \ln \fhat_t(x'_\ell)^\top g^\star(y_\ell, s_\ell)
            \end{equation*}
            \State Define empirical covariance matrix \label{line:empirical-covariance} 
            \begin{equation}\label{eqn:empirical_covariance_matrix}
            \widehat{\Sigma}_t = \lambda_t I + \sum_{\ell=1}^{t-1} g^\star(y_\ell, s_\ell) g^\star(y_\ell, s_\ell)^\top
            \end{equation}
            \State Define policy for this round \begin{equation}\label{eqn:policy_definition}
            \pi_t: x,s \mapsto  \argmax_{y \in \Ycal} \rbr{\fhat_t(x)^\top g^\star(y, s)   + b_t(y, s)}
    \end{equation}\label{line:policy_definition}
            where $$b_t(y, s) =  C'\left( \sqrt{\log(t|\mathcal{F}|/\delta)} + \sqrt{\lambda_t}B \right)\| g^\star(y, s) \|_{\widehat{\Sigma}_t^{-1}} $$ and $C'$ is determined by~\pref{eqn:upper_bound_cov_norm_single_point_main} in~\pref{lem:estimation_sigma_t_f_main} .\vspace*{0.2cm}
            \State Agent generates the response $y_t = \pi_t(x_t, s_t)$ \label{line:alg-generate-response}\vspace*{0.2cm}
             \State Teacher describes the response $x'_t \sim \PP(X \mid y_t, s_t)$\label{line:alg-teacher-response}\vspace*{0.2cm}
             \State Evaluate using a \emph{hidden reward} $r_t = \PP(x_t \mid y_t, s_t)$\vspace*{0.2cm}
        \EndFor
        \Return $\sum_{t=1}^T r_t$
    \end{algorithmic}
\end{algorithm}
The maximum likelihood estimation is routinely computed for complex function classes such as deep neural networks in practice. However, in this case, the main challenge is in defining a function class $\Fcal$ such that $f(\cdot)^\top g^\star(y,s)$ is a well-defined distribution. This question has been addressed for low-rank models~\citep{zhang2022making} and we expect the same tools to also help here. The optimization in \pref{line:policy_definition}-\ref{line:alg-generate-response} can be trivially solved when $\Ycal$ is small enough to be enumerated. When $\Ycal$ is exponentially large, this step can be challenging. One strategy can be to use a proposal distribution $q(y \mid x_t, s_t)$ to generate a set of $K$ responses, and then find the response with maximum objective value in~\pref{eqn:policy_definition}. The proposal distribution can be trained by performing MLE on historic data and modeling $y$ autoregressively. 
However, we leave a computational study with exponentially large $\Ycal$ for future work.

\section{Theoretical Analysis}

Our main result is to show \pref{alg:known_gstar} satisfies a sublinear regret bound in the realizable setting,
\begin{restatable}{assumption}{assumptionrealizability}[Realizability]\label{assumption:realizability} 
The teacher model $\PP(x|y,s) = f^\star(x)^\top g^\star(y,s)$ satisfies $f^\star \in \Fcal$ for a \emph{known} model class $\Fcal$. Moreover, all teacher models parametrized by feature maps $f \in \Fcal$ are valid distributions, i.e., $f^\top(X)^\top g^\star(y, s) \in \Delta(\Xcal)$, for any $y \in \Ycal$ and $s \in \Scal$.
\end{restatable}

We also assume the feature maps in $\Fcal$ are bounded. 
\begin{restatable}{assumption}{assumptionboundedfeature}\label{assumption:bounded_feature_maps} 
There exists a constant $B > 0$ such that for any $f \in \Fcal$ we have $\sup_{x \in \Xcal} \nbr{f(x)} \leq B$ and $\sup_{y\in \Ycal, s \in \Scal} \| g^\star(y,s)\| \leq B $.    
\end{restatable}

Our main theoretical result is a high probability upper bound for the regret of Algorithm~\ref{alg:known_gstar}. 

\begin{restatable}{theorem}{regretboundalgorithm}[Regret bound of $\algname$]\label{thm:regret-upper-bound}
When~\pref{assumption:realizability} and~\pref{assumption:bounded_feature_maps} hold and $\lambda_t = 1/t$ the regret of~\pref{alg:known_gstar} satisfies
\begin{align}
  \Reg(T)   = &\Ocal\rbr{B \sqrt{Td \log(1 + TB)}+\sqrt{Td \log(T|\Fcal|/\delta) \log(1 + TB)}}
\end{align}
with probability at least $1-3\delta$ for all $T \in \NN$.
\end{restatable}

The analysis of \pref{alg:known_gstar} in~\pref{thm:regret-upper-bound} is based on the principle of optimism. For any $(s,x,y) \in \Scal \times \Xcal \times \Ycal$ thequantity $\fhat_t(x)^\top g^\star(y,s) + b_t(y,s) $ can be understood as an estimator for the value of $\PP(x|y,s)$ where the corrective bonus $b_t(y,s)$ takes into account the accuracy of the empirical estimator $\widehat{\PP}_t(x|y,s) = \fhat_t(x)^\top g^\star(y,s)$. By definining the bonus function $b_t : \Ycal \rightarrow \mathbb{R}$ as an appropriately scaled multiple of $\| g^\star(y,s) \|_{\widehat{\Sigma}_t^{-1}}$ it overestimates $\PP(x|y,s)$. That is, for all $s \in \Scal$, $x \in \Xcal$, $y \in \Ycal$ and $t \in \mathbb{N}$,
\begin{equation}\label{eqn:optimism_main_pointwise}
\PP(x|y,s) \leq  \widehat{\PP}_t(x|y,s)   + b_t(y,s) 
\end{equation}
with probability at least $1-2\delta$. When~\pref{eqn:optimism_main_pointwise} is satisfied the policy definition of~\pref{line:policy_definition} immediately implies that
\begin{align}\label{eqn:optimism_main_max_policy} 
\max_{y \in \Ycal}    \PP(x | y,s_t) &\leq \max_{ y \in \Ycal} \widehat{\PP}_t(x|y,s_t)    + b_t(y,s_t) = \widehat{\PP}_t(x|y_t,s_t)    + b_t(y_t,s_t) 
\end{align}
To prove~\pref{eqn:optimism_main_pointwise} we develop the following supporting result,
\begin{restatable}{lemma}{mainsupportinglemma}\label{lem:estimation_sigma_t_f_main}
When~\pref{assumption:realizability} and~\pref{assumption:bounded_feature_maps} hold, then with probability at least $1-2\delta$ we have:
\begin{equation}\label{eqn:upper_bound_cov_norm_single_point_main}
\sup_{x \in \Xcal}\nbr{f^\star(x) - \widehat{f}_t(x)}_{ \widehat{\Sigma}_t}  \leq C'\left( \sqrt{\log(t|\Fcal|/\delta)} + \sqrt{\lambda_t}B\right)
\end{equation}
for all $t \in \NN$ simultaneously. 
\end{restatable}
This result provides a bound for the maximum error in the estimation of $f^\star(x)$ as measured by the data norm $\| \cdot \|_{\widehat{\Sigma}_t}$. It will prove crucial in bounding the error of the empirical models $\widehat{\PP}_t(x|y,s)$. The detailed version of this result can its proof can be found in~\pref{lem:estimation_sigma_t_f} in~\pref{app:upper-bound}. %

We denote as $\Ecal$ the event that Equation~\ref{eqn:upper_bound_cov_norm_single_point_main} holds. In this case, we can upper bound the prediction error of the empirical model $\widehat{\PP}_t(x|y,s) =  \fhat_t(x)^\top g^\star(y,s)$ for all $(s,x, y) \in \Scal \times \Xcal \times \Ycal$. 

\begin{align}
\abr{\widehat{\PP}_t(x|y,s) - \PP(x|y,s)} &= \left| \rbr{f^\star(x) - \fhat_t(x)}^\top g^\star(y,s) \right|  
 &\stackrel{(i)}{\leq}  \nbr{f^\star(x)- \fhat_t(x )}_{\widehat{\Sigma}_t}  \nbr{g^\star(y,s)}_{ \widehat{\Sigma}_t^{-1}}
&\stackrel{(ii)}{\leq}  b_t(y,s) \label{eqn:bounding_prediction_error}
\end{align}

where $(i)$ holds because for all $v,w \in \mathbb{R}^d$ and invertible $\Sigma \in \mathbb{R}^{d\times d}$, the Cauchy-Schwartz inequality implies $\langle v, w\rangle = \langle \Sigma^{1/2} v , \Sigma^{-1/2} w \rangle \leq \| v \|_{\Sigma} \| w \|_{\Sigma^{-1}}$ and $(ii)$ by upper bounding $\nbr{f^\star(x)- \fhat_t(x )}_{\widehat{\Sigma}_t} $ using the RHS of Equation~\ref{eqn:upper_bound_cov_norm_single_point_main}. 

These are the necessary ingredients to finalize our sketch of~\pref{thm:regret-upper-bound}. When $\Ecal$ holds the following inequalities are satisfied,
\begin{align*}
&\Reg(T) = \sum_{t=1}^T \rbr{\PP(x_t \mid \pi^\star(x_t), s_t) - \PP(x_t \mid \pi_t(x_t), s_t)}\\
&\qquad\stackrel{(a)}{\leq} \sum_{t=1}^T \fhat_t(x)^\top g^\star(y_t, s_t) - f^\star(x_t)^\top g^\star(y_t, s_t)+ b_t(y_t, s_t)\\
&\qquad\stackrel{(b)}{\leq}  \sum_{t=1}^T 2b_t(y_t, s_t)
\end{align*}
where $(a)$ is a consequence of Optimism (\pref{eqn:optimism_main_max_policy}). And inequality~$(b)$ of the prediction error bound from~\pref{eqn:bounding_prediction_error}. What we have managed to achieve at this point is to upper bound the regret by a sum of estimation errors along the features of the responses played by the algorithm at each time-step. Finally, substituting the definition of the bonus terms and invoking a standard sum of inverse norms bound from the linear bandits literature (see for example Proposition~$3$ in~\cite{pmlr-v130-pacchiano21a} and Proposition~\ref{prop:determinant_result} in Appendix~\ref{appendix::useful_lemmas}) 
\begin{align*}
\sum_{t=1}^T b_t(y_t) &\leq \mathcal{O}\left( \sqrt{\log(T |\Fcal| / \delta ) } \sum_{t=1}^T \| g^\star(y_t, s_t)\|_{\widehat{\Sigma}_t^{-1}} \right)\leq    \mathcal{O}\left(  \sqrt{ Td \log(T|\mathcal{F}| /\delta) \log(1+TB)  }   \right)
\end{align*}
This finalizes the proof sketch of Theorem~\ref{thm:regret-upper-bound}. These results rely on the assumption that $g^\star $ is known. Removing that assumption yields a substantially harder problem as it makes it more difficult to leverage the linearity structure. Although this scenario can be dealt with by deriving algorithms and bounds depending on statistical capacity measures such as the eluder dimension~\citep{russo2013eluder} for the combined $f(x)^\top g(y,s)$ model class we leave the derivation of a sharper analysis of this setting for future work.

\section{Empirical Study}

We evaluate $\algname$ in two settings. The first is a synthetic task that satisfies the low-rank teacher setting and all our assumptions and is designed to provide a proof of concept of $\algname$. The second setting is a grounded setting with real images, natural language instructions, and where the teacher model is not low-rank. Our goal with these experiments is not to present challenging settings for exploration, but to show how various components of $\algname$ can be implemented empirically. Our second experiment also evaluates whether insights from $\algname$ carry over to more realistic settings even where our assumptions are violated.

\subsection{Evaluating on a Synthetic Task}

\paragraph{Environment.} For a given intrinsic dimension $d$, instruction size $|\Xcal|$ and response size $|\Ycal|$, we randomly initialize two matrices $F \in \RR^{|\Xcal| \times d}$ and $G \in \RR^{d \times |\Ycal|}$. We ignore the context in this setting by defining $\Scal$ as a singleton $\{s_0\}$. We define $\Xcal = \sbr{|\Xcal|}$ and $\Ycal = \sbr{|\Ycal|}$ and so an instruction $x$ and a response $y$ are positive integers. We define these matrices by first initializing them with values sampled iid with standard Gaussian distribution. We then take their exponent and divide by a temperature coefficient $\tau$. We then normalize $F$ and $G$ row-wise such that $FG \in \RR^{|\Xcal| \times |\Ycal|}$ is a stochastic matrix whose columns sum to 1. For a given $(x, y) \in \Xcal \times \Ycal$, we view the matrix entry $(FG)_{xy}$ as denoting the value of the teacher distribution $\PP(X=x \mid Y=y, S=s_0)$. We can view the $x^{th}$ row of $F$ and the $y^{th}$ column of $G$ as denoting $f^\star(x)$ and $g^\star(y)$ respectively.

\paragraph{Baselines.} We evaluate the following baselines. \textit{Random}: 
the agent takes uniformly random actions. \textit{$\epsilon$-Greedy}: the agent performs maximum-likelihood estimation on the historic data to learn an estimate $\fhat_t$ similar to $\algname$; however, unlike $\algname$, the exploration is not performed using elliptic bonus but using $\epsilon$-greedy, where with $\epsilon$ probability a random action is taken and with the remaining probability, we take the greedy action $\arg\max_{y \in \Ycal} \fhat_t(x_t)^\top g^\star(y)$. \textit{Greedy}: This is same as $\epsilon$-Greedy with $\epsilon=0$ and only exploits based on historic data. We tune the hyperparameters $\lambda$ and $C'$ for $\algname$ and $\epsilon$ for $\epsilon$-greedy using grid search.

\begin{figure}[!h]
    \centering
    \includegraphics[scale=0.55]{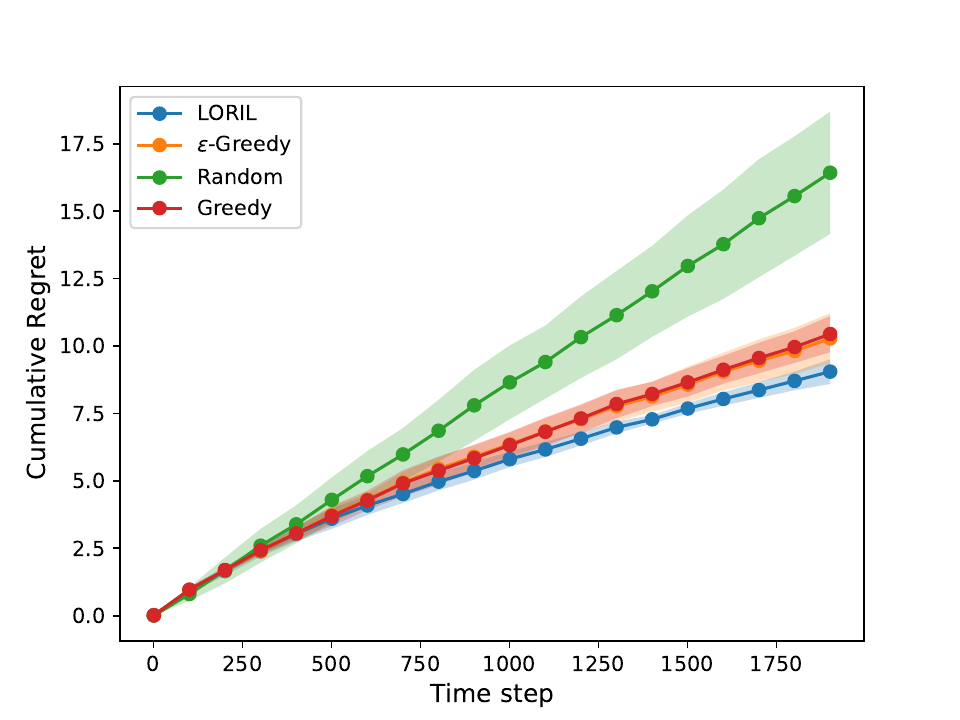}
    \caption{Comparison of $\algname$ against baselines on the controlled task. We run each baseline 3 times and report the average. The shaded areas show the standard deviation.}
    \label{fig:first-exp-synthetic}
\vspace{-.4cm}
\end{figure}

\paragraph{Model and Optimization.} We model $f \in \Fcal$ as $f(x) = \rbr{\frac{\exp(\theta_{xi})}{\sum_{x' \in \Xcal} \exp(\theta_{x'i})}}^d_{i=1}$, where $\rbr{\theta_{xi}}_{x \in \Xcal, i \in [d]}$ are the parameters that we train. For any $f \in \Fcal$, we can verify that $f(x)^\top g^\star(y)$ is a valid conditional distribution over $x$ given $y$. We perform maximum likelihood estimation using Adam optimization.

\paragraph{Results.} \pref{fig:first-exp-synthetic} shows cumulative regret over time steps for $\algname$ and baselines. We ran each experiment 3 times with different seeds. We select hyperparameters for each algorithm based on the mean final regret. We can see that $\algname$ performs better than all baselines achieving the best regret which is 12.3\% smaller than than the next best baseline. Improvements over the greedy baseline show that exploration helps, whereas improvements over $\epsilon$-greedy show that using elliptic bonus for exploration provides better regret bounds.

\subsection{Evaluation on an Image Selection Task}

We evaluate $\algname$ on an image classification task where the true model does not admit a low-rank decomposition. In reinforcement learning, it has been found that using the elliptic bonus for exploration is helpful in real-world settings where low-rank assumption doesn't hold~\cite{henaff2022exploration}. Our goal in this subsection is to test if a similar result holds for our setting.

\begin{figure}[!h]
    \centering
    \includegraphics[scale=0.55]{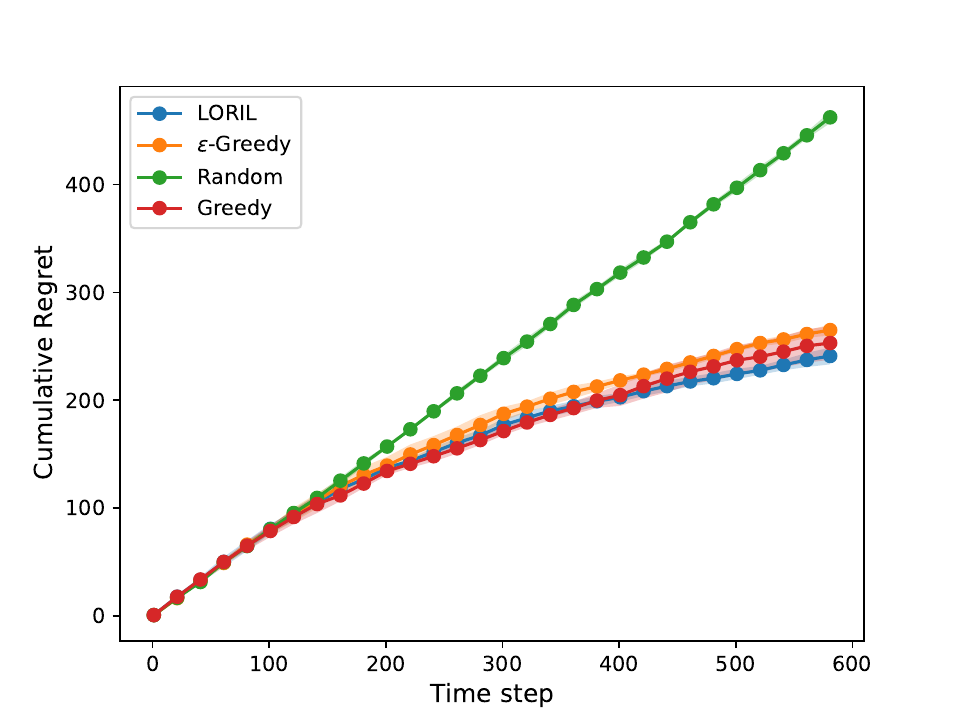}
    \caption{Results on the image classification task. We run each baseline 5 times and report the average performance. The shaded areas show the standard deviation.}
    \label{fig:second-exp}
    \vspace{-.4cm}
\end{figure}

\paragraph{Environment.} The instruction space $\Xcal$ is in natural language where for a given $x \in \Xcal$, we denote the $i^{th}$ token by $x_i$. The agent has an action space with $|\Ycal|=K$ actions. In each round, the world assigns an image of an object to each action, and the agent is given a natural language instruction describing the image that the agent should select. The agent has a non-trivial context $s \in \Scal$ that contains the identity of the object assigned to each action in the given round. No two actions have the same image but the image associated with each action can change across rounds. We sample an object by picking an object of a given type and a given color from a set of types and colors. The instruction space $\Xcal$ is in natural language as in our motivating example in~\pref{fig:flowchart}. We use a set of templates to generate instructions for describing a given image using the object type and color.

\paragraph{Model.} We model a function class $\Hcal = \{h: \Ycal \times \Scal \rightarrow \Delta(\Xcal)\}$ using a deep neural network. Given an action $y$ and context $s$, we encode the image associated with the action with an encoding $g^\star(y, s) \in \RR^d$. We model $g^\star$ as a 3-layer convolutional neural network with LeakyReLu activation. In this setting, we don't assume that the environment provides the $g^\star$. Instead, we train $g^\star$ using an autoencoder objective and a set of offline images sampled from the environment. Alternatively, we could have used a pre-trained image representation model such as ResNet~\citep{he2016deep} or CLIP~\citep{radford2021learning}.

We use the encoding $g^\star(y, s)$ to generate a distribution over texts $x = (x_1, \cdots, x_n)$ using a two-layer Gated Recurrent Unit (GRU). Specifically, we apply a fully connected layer to $g^\star(y, s)$ to reshape it to an appropriate size and use it to initialize the hidden state of the GRU for all layers.

\paragraph{Results.} \pref{fig:second-exp} shows the results. Similar to our previous experiment, $\algname$ performs better than baselines, achieving 4.8\% less regret than the next best baseline, even though the setting does not admit a low-rank structure. We also note that these tasks were not designed to present a challenging scenario for exploration, and consequently the gains relative to baselines are smaller.

\section{Related Work}
\label{sec:related-work}

\paragraph{Provably-efficient Interactive Learning.}  The ubiquitous nature of interactive learning has resulted in significant attention devoted to its theoretical understanding. \pref{proto:protocol} superficially resembles a contextual bandit problem but stands in contrast with the scenario where the learner receives a (possibly noisy) reward signal after taking an action in a given context. The key difference is that while in the contextual bandit setting the feedback equals an unbiased sample of the reward corresponding to the arm and context, in our setting the feedback is produced from a conditional distribution of instructions that does not immediately relate to the reward. Thus, it is not possible to immediately adapt a contextual bandit algorithm to provide regret bounds for \pref{proto:protocol}. There is a vast literature dedicated to developing sublinear regret algorithms for contextual bandit problems. Early efforts to incorporate contextual information into bandit problems led to the development of algorithms such as LinUCB~\citep{chu2011contextual}, OFUL~\citep{abbasi2011improved}, and Linear Thompson Sampling~\citep{agrawal2013thompson}, for the setting when there is a linear relationship between the context and the reward.  %
A long line of work has also focused on studying guarantees for imitation learning~\citep{ross2011reduction,rashidinejad2021bridging}, and policy optimization~\citep{kearns2002near,auer2006logarithmic,azar2017minimax,foster2021statistical}. More recently, there has been a focus on developing statistically efficient RL algorithms with function approximation~\citep{misra2020kinematic,jin2021bellman,foster2021statistical}. Our work focuses on provable learning similar to these methods and uses similar tools for analysis, but focuses on a novel learning setting with a different type of feedback than IL and RL.

\paragraph{Low-rank Interactive Learning.} Low-rank models have been studied in bandit settings~\citep{abbasi2011improved}, contextual bandit settings~\citep{chu2011contextual}, and in more general multi-step reinforcement learning~\citep{jin2020provably,agarwal2020flambe}. One of the appeals of low-rank models is that they can generalize tabular MDPs and provide a way to study function approximation settings which is standard in empirical studies. This is also our motivation for studying low-rank models. Further, low-rank models are one of the most expressive settings for which both statistically and computationally efficient algorithms exist.

\paragraph{Learning using Hindsight Feedback.} Several different works have found it advantageous to use hindsight feedback to convert a failed example into a positive example by relabeling it with a different goal (or in our case instruction)~\citep{andrychowicz2017hindsight,li2020generalized,nair2018overcoming}. These approaches typically solve goal-conditioned RL where a failed trajectory is labeled with its final state as the goal. However, these approaches focus on empirical performance and do not provide regret bounds.

\paragraph{Instruction Following.} The task of developing agents that can follow natural language instructions has received significant attention since the early days of AI~\citep{winograd1972understanding}. Several approaches have developed methods that train these systems using imitation learning~\citep{mei2016listen} and reinforcement learning~\citep{misra2018mapping,hill2021grounded}. Training agents with hindsight instruction labeling has been previously explored for instruction following in~\cite{nguyen2021interactive,fried2018speaker}. The main focus of these results is on empirical performance and they either provide no theoretical analysis, or in the case of~\cite{nguyen2021interactive} only provide asymptotic analysis. In contrast, we provide the first finite-sample regret bounds for learning from instruction labeling.

\section{Conclusion}

In this work we define a \emph{formal} interactive learning setup for hindsight instruction and initiate its theoretical understanding. Among other things we present a lower bound indicating that hindisght instruction learning in the general case can be statistically intractable, thus implying the necessity of imposing structural conditions for statistically efficient learning with hindsight feedback. We present an algorithm $\algname$ that has no-regret when the underlying teacher distribution has low-rank. The regret of $\algname$ scales $\widetilde{\Ocal}\rbr{\sqrt{T}}$ with the horizon and only depends on the rank of the distribution and does not depend on the size of the agent's response space or instruction space. We finalize our work with an experimental demonstration of $\algname$ in a variety of synthetic and grounded scenarios. This work represents a first exploration of the hindsight instruction setup and therefore many exciting research directions remain open. Chief among them is to design provably efficient algorithms for hindsight instruction under less restrictive function approximation assumptions and that are also computationally efficient. We foresee that the algorithmic framework introduced by the Decision Estimation Coefficient literature ~\citep{foster2023tight,foster2021statistical} can serve as the basis of the development of algorithms for hindsight instruction that are both computationally tractable and statistically efficient and that can lead to practical impact in scenarios such as training language models and robotics.

\section*{Acknowledgements}
We thank Dylan J. Foster, Nan Jiang, Akshay Krishnamurthy, and John Langford for interesting discussions. AP's participation in this work was supported in part by funding from the Eric and Wendy Schmidt Center at the Broad Institute of MIT and Harvard.

\newpage
\bibliographystyle{plainnat}
\bibliography{arxiv}

\newpage
\appendix
\onecolumn

\section{Lower Bound Proofs}\label{section::lower_bound_appendix}

\lemmalowerbound*

\begin{proof}
Throughout the proof we'll use the notation $\epsilon = \sqrt{K/T}$ so that problem $W_i$ has distribution $\PP(X, S) = \mathrm{Uniform}((A, s_o),(B, s_o))$ and 
\begin{equation*}
    \PP_{W_i}(X | y_j) =   1/2 + \epsilon \cdot \mathbf{1}(j = i)\cdot(1-2\cdot\mathbf{1}(X = B)) .
\end{equation*}

Let's start by defining the empty problem $W_\emptyset$ as
\begin{align*}
\PP(X, S) &= \mathrm{Uniform}((A, s_o),(B, s_o))\\
\mathbb{P}(A | y_i) &= \mathbb{P}(B | y_i) = 1/2, \quad \forall i
\end{align*}

Let's consider an arbitrary algorithm for $\mathbb{A}$ for learning with hindsight labeling and consider its interaction with problem $W$. We'll use the notation $\mathbb{P}_{W, \mathbb{A}}$ and $\mathbb{E}_{W, \mathbb{A}}$ to denote the measure and expectations induced by problem $W$ and algorithm $\mathbb{A}$. 

First, we consider algorithm $\mathbb{A}$'s interaction with problem $W_\emptyset$. 

Let's consider an arbitrary algorithm for $\mathbb{A}$ for learning with hindsight labeling and its interactions with $W_\emptyset$. 

Let $n_i(T) = \sum_{t=1}^T \mathbf{1}( y_t = i) $ denote the (random) number of times the learner selected $y_t = i$ from time $1$ to $T$.

We'll use the notation $\ell_1, \cdots, \ell_K$ to denote the ordering of indices in $[K]$ such that,
\begin{equation*}
    \mathbb{E}_{\PP_{W_\emptyset, \mathbb{A}}}\left[n_{\ell_1}(T)\right]  \leq \cdots \leq     \mathbb{E}_{\PP_{W_\emptyset, \mathbb{A}}}\left[n_{\ell_K}(T)\right] .
\end{equation*}
Notice that for all $j \leq \lfloor K/2 \rfloor$,
\begin{equation}\label{equation::expected_num_pulls}
  \mathbb{E}_{\PP_{W_\emptyset, \mathbb{A}}}\left[n_{\ell_1}(T)\right]  \leq \frac{2T}{K}, \forall \ell \leq \lfloor K/2 \rfloor
\end{equation}

Let's start by noting that for any $W_i$,  the KL distance between $\mathbb{P}_{W_\emptyset, \mathbb{A}}$ and $\mathbb{P}_{W_i, \mathbb{A}}$ for $i \in \{ \ell_j \}_{j \in \left[\lfloor K/2 \rfloor\right]}$ satisfies the bound

\begin{align*}
\mathrm{KL}\left( \mathbb{P}_{W_\emptyset, \mathbb{A}} \parallel \mathbb{P}_{W_i, \mathbb{A}} \right) &= \mathbb{E}_{W_\emptyset, \mathbb{A}}\left[ \sum_{t=1 }^T \mathrm{KL}( \mathbb{P}_{W_\emptyset}( X' | Y_t ) \parallel \mathbb{P}_{W_i}(X'| Y_t ) ) \right] \\
&= \mathbb{E}_{W_\emptyset, \mathbb{A}}\left[ \sum_{t=1 }^T \mathrm{KL}( \mathbb{P}_{W_\emptyset}( X' | Y_t = i ) \parallel \mathbb{P}_{W_i}(X'| Y_t = i ) )\mathbf{1}(Y_t = i) \right] \\
&=  \mathbb{E}_{W_\emptyset, \mathbb{A}}\left[ n_i(T)\right] \mathrm{KL}\left( \mathrm{Ber}(1/2) \parallel \mathrm{Ber}(1/2-\epsilon) \right) \\
&\stackrel{(i)}\leq  \mathcal{O}\left( \frac{T\epsilon^2}{K} \right)
\end{align*}
Where $(i)$ is implied by inequality~\ref{equation::expected_num_pulls} for all $\ell_j$ with $j \in \left[\lfloor K/2 \rfloor\right]$. 

Define $\mathcal{E}_i = \left\{  \sum_{t=1}^T \mathbf{1}( Y_t = i, X_t = A) +  \mathbf{1}( Y_t \neq i, X_t = B) \geq \frac{7T}{8} \right \} $. When interacting with world $W_i$ the event $\mathcal{E}_i$ corresponds to the event where $\mathbb{A}$ makes the right decisions in at least $3T/4$ time-steps. 

The complement event $\mathcal{E}_i^c =\{  \sum_{t=1}^T \mathbf{1}( Y_t = i, X_t = A) +  \mathbf{1}( Y_t \neq i, X_t = B) < \frac{7T}{8}  \}$ corresponds to the event where $\mathbb{A}$ makes the correct decisions in at most $7T/8$ time-steps.

By the Huber-Bretagnolle inequality, all $i \in \{ \ell_j \}_{j \in \left[\lfloor K/2 \rfloor\right]}$ satisfy
\begin{align*}
\PP_{W_\emptyset, \mathbb{A}}(\mathcal{E}_i) + \PP_{W_i, \mathbb{A}}(\mathcal{E}_i^c) &\geq \exp\left(- \mathrm{KL}\left( \mathbb{P}_{W_\emptyset, \mathbb{A}} \parallel \mathbb{P}_{W_i, \mathbb{A}} \right)\right)\\
&\geq \mathcal{O}\left( -T\epsilon^2/K \right).
\end{align*}
since $\epsilon = \sqrt{K/T}$, we have 
$$\PP_{W_\emptyset, \mathbb{A}}(\mathcal{E}_i) + \PP_{W_i, \mathbb{A}}(\mathcal{E}_i^c) \geq 1/e \text{ for all $i \in \{ \ell_j \}_{j \in \left[\lfloor K/2 \rfloor\right]}$}.$$

Let $n_X(T) = \sum_{t=1}^T \mathbf{1}( X_t = X) $ denote the (random) number of times the learner selected $X_t = X$ for $X \in \{A,B\}$ from time $1$ to $T$.

We now define $\mathcal{E}_{\mathrm{good}} = \{ n_B(T) \leq \frac{5T}{8}\}$. If $T \geq 256\log(2e)$ Proposition~\ref{lemma::supporting_hoeffding_lemma} applied with $\alpha = 1/8$ and $\delta = 1/2e$ implies that,
\begin{equation*}
\PP_{W_\emptyset}(\mathcal{E}_{\mathrm{good}} )  \geq 1-1/2e 
\end{equation*}

Define $\widetilde{\mathcal{E}}_i= \mathcal{E}_i \cap \mathcal{E}_{\mathrm{good}}$. For all $i \in  \{ \ell_j \}_{j \in \left[\lfloor K/2 \rfloor\right]}$,
\begin{align*}
1/e \leq \PP_{W_\emptyset, \mathbb{A}}(\mathcal{E}_i) + \PP_{W_i, \mathbb{A}}(\mathcal{E}_i^c) &= \PP_{W_\emptyset, \mathbb{A}}(\mathcal{E}_i\cap \mathcal{E}_{\mathrm{good}}^c) +\PP_{W_\emptyset, \mathbb{A}}(\widetilde{\mathcal{E}}_i) +  \PP_{W_i, \mathbb{A}}(\mathcal{E}_i^c)\\
&\leq \PP_{W_\emptyset, \mathbb{A}}( \mathcal{E}_{\mathrm{good}}^c) +\PP_{W_\emptyset, \mathbb{A}}(\widetilde{\mathcal{E}}_i) +  \PP_{W_i, \mathbb{A}}(\mathcal{E}_i^c) \\
&\leq 1/2e + \PP_{W_\emptyset, \mathbb{A}}(\widetilde{\mathcal{E}}_i) +  \PP_{W_i, \mathbb{A}}(\mathcal{E}_i^c)
\end{align*}

And therefore $\PP_{W_\emptyset, \mathbb{A}}(\widetilde{\mathcal{E}}_i) +  \PP_{W_i, \mathbb{A}}(\mathcal{E}_i^c) \geq 1/2e$. Thus,
\begin{equation}\label{equation::sum_all_lower_bounds}
    \sum_{i \in \{ \ell_j \}_{j \in [\lfloor K/2 \rfloor]}}  \PP_{W_\emptyset, \mathbb{A}}(\widetilde{\mathcal{E}}_i) + \PP_{W_i, \mathbb{A}}(\mathcal{E}_i^c) \geq \lfloor K/2 \rfloor /2e
\end{equation}
Notice that when $\widetilde{\mathcal{E}}_i$ it follows that $\sum_{t=1}^T \mathbf{1}(Y_t = i, X_t = A) + \frac{5T}{8}\geq   \sum_{t=1}^T \mathbf{1}( Y_t = i, X_t = A) +  \mathbf{1}( Y_t \neq i, X_t = B)\geq  \frac{7T}{8}$ and therefore $\sum_{t=1}^T \mathbf{1}(Y_t = i, X_t = A) \geq \frac{T}{4}$. This in turn implies that when $\widetilde{\mathcal{E}}_i$ holds then $\sum_{t=1}^T \mathbf{1}( Y_t \neq i, X_t = A) + \mathbf{1}(X_t = B) \leq \frac{3T}{4}$.

Notice that $\widetilde{\mathcal{E}}_i \cap \widetilde{\mathcal{E}}_j = \emptyset$ for all $i \neq j$. This is because for all $j \neq i$, when $\widetilde{\mathcal{E}}_i$ holds,
\begin{equation*}
\sum_{t=1}^T \mathbf{1}(Y_t = j, X_t = A) + \mathbf{1}(Y_t \neq j, X_t = B) \leq \sum_{t=1}^T \mathbf{1}( Y_t \neq i, X_t = A) + \mathbf{1}(X_t = B) \leq \frac{3T}{4} < \frac{7T}{8} 
\end{equation*}
Since $\widetilde{\mathcal{E}}_i \cap \widetilde{\mathcal{E}}_j = \emptyset$ the sum $ \sum_{i \in \{ \ell_j \}_{j \in [\lfloor K/2 \rfloor]}}  \PP_{W_\emptyset, \mathbb{A}}(\widetilde{\mathcal{E}}_i) \leq 1$. Thus Equation~\ref{equation::sum_all_lower_bounds} implies,
\begin{equation*}
    \sum_{i \in \{ \ell_j \}_{j \in [\lfloor K/2 \rfloor]}}  \PP_{W_i, \mathbb{A}}(\mathcal{E}_i^c) \geq \lfloor K/2 \rfloor /2e - 1
\end{equation*}
And therefore there is an index $\hat{i} \in \{ \ell_j \}_{j \in [\lfloor K/2 \rfloor]}$ such that,
\begin{equation*}
 \PP_{W_{\hat i}, \mathbb{A}}(\mathcal{E}_{\hat i}^c)  \geq  \frac{1}{2e} - \frac{1}{\lfloor K/2 \rfloor} \stackrel{(i)}{\geq} \frac{1}{4e}.
\end{equation*}
where inequality $(i)$ holds because $K \geq 8e$. 

when $\mathcal{E}_{\hat i}^c = \{  \sum_{t=1}^T \mathbf{1}( Y_t = i, X_t = A) +  \mathbf{1}( Y_t \neq i, X_t = B) < \frac{7T}{8}  \}$ holds, 
\begin{equation}\label{equation::complement_event_consequences_pulls}
\sum_{t=1}^T \mathbf{1}( Y_t \neq i, X_t = A) +  \mathbf{1}( Y_t = i, X_t = B) \geq \frac{T}{8} 
\end{equation}
also holds. When $\mathcal{E}_{\hat i}^c$ is satisfied and therefore Equation~\ref{equation::complement_event_consequences_pulls} is satisfied, the regret can be lower bounded by $\epsilon T/8 = \frac{\sqrt{KT}}{8}= \Omega(\sqrt{KT})$ since $\epsilon = \sqrt{ K/T}$. We conclude that,
\begin{equation*}
\Reg(T) \geq \frac{\sqrt{KT}}{8}
\end{equation*}
with probability $\mathbb{P}_{W_{\hat{i}}, \mathbb{A}}( \mathcal{E}_{\hat i}^c) \geq \frac{1}{4e} $. 
\end{proof}

\expectedregretlowerbound*

\begin{proof}

Lemma~\ref{lemma::lower_bound} implies there exists at least one problem $W_{\hat{i}}$ such that $\Reg(T)  \geq \frac{\sqrt{KT}}{8}$ with probability at least $1/4e$. Let's call this event $\mathcal{E}$. Therefore since $\sum_{t=1}^T \max_{y \in Y} \PP(  X_t | y, s_o)  - \PP(X_t | Y_t, s_o ) \geq 0$ with probability one, 

\begin{equation*}
   \overline{\Reg}(T) =  \mathbb{E}_{W_{\hat{i}}, \mathbb{A}}\left[ \sum_{t=1}^T \max_{y \in Y} \PP(  X_t | y, s_o)  - \PP(X_t | Y_t , s_o)\right] \geq \mathbb{P}_{W_{\hat{i}}, \mathbb{A}}( \mathcal{E})\cdot \frac{\sqrt{KT}}{8} \geq  \Omega(\sqrt{KT}).
\end{equation*} 
 
\end{proof}

\begin{proposition}\label{lemma::supporting_hoeffding_lemma}
Let $\delta \in (0,1)$, $\alpha \in (0,1/2)$ and $\{X_i\}_{i=1}^T$ be $T$ i.i.d. random variables sampled from $\mathrm{Ber}(1/2)$. It $T \geq \frac{4}{\alpha^2}\log(1/\delta)$ then $\sum_{i=1}^T X_i \leq (\frac{1}{2}+\alpha)T$ with probability at least $1-\delta$.  Similarly if $T \geq \frac{4}{\alpha^2}\log(1/\delta)$ then $\sum_{i=1}^T X_i \geq (\frac{1}{2}-\alpha)T$ with probability at least $1-\delta$. 
\end{proposition}

\begin{proof}

Let $\widehat{S}= \sum_{i=1}^T X_i$ be the sum of the outcomes $\{X_i\}_{i \in [T]}$. Hoeffding inequality implies

$$ \widehat{S} - T/2   \leq 2\sqrt{T \log(1/\delta) } $$

with probability at least $1-\delta$. Thus, as long as $2\sqrt{T \log(1/\delta) }  \leq \alpha T$, (i.e. $\frac{4}{\alpha^2} \log(1/\delta) \leq T$) 
we have $\widehat{S} \leq (\frac{1}{2} + \alpha)T$. 

The reverse inequality can be derived using the same argument applied to the inequality $T/2 - \widehat{S} \leq 2\sqrt{T \log(1/\delta)}$ with probability at least $1-\delta$.

\end{proof}

\section{Regret Bounds for $\algname$ in low-rank distribution setting}
\label{app:upper-bound}

In this section, we provide a regret bound for $\algname$. We first enumerate the assumptions.

\assumptionrealizability*

\assumptionboundedfeature*

Let us consider the $t^{th}$ round. The empirical covariance matrix is given by $\widehat\Sigma_t$ where
\begin{equation*}
\widehat{\Sigma}_t = \sum_{l=1}^{t-1} g^*(y_l, s_l)g^\star(y_l, s_l)^\top + \lambda_t \II.
\end{equation*}
for a regularizer value $\lambda_t >0$. It is easy to verify that $\widehat\Sigma_t$ is a positive definite matrix since $\lambda_t \II$ is a positive definite matrix and $C_t = \sum_{l=1}^{t-1} g^*(y_l, s_l)g^\star(y_l, s_l)^\top$ is a positive semidefinite matrix as it is a symmetric matrix and for any $v \in \RR^d$ we have $v^\top C_t v = \sum_{l=1}^{t-1} v^\top g^\star(y_l, s_l) g^\star(y_l, s_l)^\top v = \sum_{l=1}^{t-1} \nbr{g^\star(y_l, s_l)^\top v}^2_2 \ge 0$. As $\widehat\Sigma_t$ is positive definite its inverse $\widehat\Sigma^{-1}_1$ and exists and is also a positive definite matrix.\footnote{This is trivial to show as $v^\top \Sigma^{-1} v = (v^\top \Sigma^{-1})\Sigma (\Sigma^{-1} v) > 0$ if $\Sigma^{-1}v \ne 0$ as $\Sigma$ is positive definite, further $\Sigma^{-1}v = 0 \Leftrightarrow v = 0$. Therefore, if $v \ne 0$, then $v^\top \Sigma^{-1} v > 0$ and vice versa.} Finally, it can be shown that one can also define the square root of the matrix $\Sigma^{1/2}_t$ and that of its inverse $\Sigma^{-1/2}_t$ and these are symmetric and positive definite as well.

Let $\widehat\PP_t(x \mid y,s) = \fhat_t(x)^\top g^\star(y,s)$ be the model estimated in round $t$ by maximum likelihood estimation. Given any $s,x, y \in \Scal \times \Xcal \times \Ycal$, the following important inequality holds,
\begin{align}
\label{eqn:holder_ineq_probabilities_known_gstar}
\abr{\PP(x|y,s) - \widehat{\PP}_t(x|y,s)} &= \abr{\rbr{f^\star(x) - \fhat_t(x)}^\top g^\star(y,s)}, \nonumber\\
&= \abr{\rbr{\Sigma^{1/2}_t\rbr{f^\star(x) - \fhat_t(x)}}^\top \rbr{\Sigma^{-1/2}_t g^\star(y,s)}},\nonumber\\
&\le \sqrt{\rbr{f^\star(x) - \fhat_t(x)}^\top \Sigma^{1/2}_t \cdot \Sigma^{1/2}_t \rbr{f^\star(x) - \fhat_t(x)}} \cdot \sqrt{g^\star(y,s)^\top \Sigma_t^{-1/2} \cdot \Sigma^{-1/2}_t g^\star(y,s)},\nonumber\\
&= \nbr{f^\star(x) - \fhat_t(x)}_{\Sigma_t} \cdot \nbr{g^\star(y,s)}_{\Sigma^{-1}_t},
\end{align}
where the second last step uses Cauchy-Schwarz inequality. This inequality allows us to bound the error in the estimated model for a given $y$ in terms of the error based on historical data given by $\nbr{f^\star(x) - \fhat_t(x)}_{\Sigma_t}$ and the novelty of the given input $\nbr{g^\star(y,s)}_{\Sigma^{-1}_t}$. We want to bound the two terms in RHS of~\pref{eqn:holder_ineq_probabilities_known_gstar}.

First we'll prove the following result,
\begin{lemma}\label{lem:estimation_sigma_t_f} When~\pref{assumption:realizability} and~\pref{assumption:bounded_feature_maps}, then with probability at least $1-2\delta$ we have:
\begin{equation}\label{eqn:upper_bound_cov_norm_single_point}
\sup_{x \in \Xcal}\nbr{f^\star(x) - \widehat{f}_t(x)}_{\widehat \Sigma_t}  \leq  (32C)^{1/4} \sqrt{\log(t|\Fcal|/\delta)} + 2\sqrt{\lambda_t}B
\end{equation}
for all $t \in \NN$ simultaneously. 
\end{lemma}

\begin{proof} 
\begin{align}
    \label{equation::upper_bound_sum_norms}
   \sup_{x \in \Xcal} \| f^*(x) - \widehat{f}_t(x) \|^2_{\widehat\Sigma_t}  
   &= \sup_{x \in \Xcal} \rbr{\sum_{\ell=1}^{t-1} (f^\star(x) - \fhat_t(x) )^\top g^\star(y_\ell, s_\ell) \cdot g^\star(y_\ell, s_\ell)^\top  (f^\star(x) - \fhat_t(x) ) + \lambda_t \nbr{f^\star(x) - \fhat_t(x)}^2_2}  \notag\\
   &= \sup_{x \in \Xcal} \rbr{\sum_{\ell=1}^{t-1} \rbr{f^\star(x)^\top g^\star(y_\ell, s_\ell) - \fhat_t(x)^\top g^\star(y_\ell, s_\ell)}^2 + \lambda_t \nbr{f^\star(x) - \fhat_t(x)}^2_2}  \notag\\
   &\le \underbrace{\sum_{\ell=1}^{t-1}  \sup_{x \in \Xcal} \rbr{f^\star(x)^\top g^\star(y_\ell, s_\ell) - \fhat_t(x)^\top g^\star(y_\ell, s_\ell)}^2}_\text{$\defeq U_t$ first term} + \lambda_t  \underbrace{\sup_{x \in \Xcal} \nbr{f^\star(x) - \fhat_t(x)}^2_2}_\text{$\defeq V_t$ second term}  \notag
\end{align}
We first bound the second term $V_t$ as:
\begin{align*}
    \sup_{x \in \Xcal} \nbr{f^\star(x) - \fhat_t(x)}^2_2 &= \sup_{x \in \Xcal} \rbr{\nbr{f^\star(x)}^2_2 + \nbr{\fhat_t(x)}^2_2 - 2f^\star(x)^\top \fhat_t(x)}, \\
    &\le \sup_{x \in \Xcal} \nbr{f^\star(x)}^2_2 + \sup_{x \in \Xcal}\nbr{\fhat_t(x)}^2_2 + 2 \sup_{x \in \Xcal} \abr{f^\star(x)^\top \fhat_t(x)},\\
    &\le \sup_{x \in \Xcal} \nbr{f^\star(x)}^2_2 + \sup_{x \in \Xcal}\nbr{\fhat_t(x)}^2_2 + 2\sup_{x \in \Xcal} \nbr{f^\star(x)}_2 \nbr{\fhat_t(x)}_2,\\
    &\le 4B^2,
\end{align*}
where we use the assumption that $\sup_{x \in \Xcal} \nbr{f(x)}_2 \le B$ for any $f \in \Fcal$ and that $f^\star, \fhat_t \in \Fcal$ and the second last step uses Cauchy-Schwarz inequality. 

We now bound the first term $U_t$. For a given $f \in \Fcal$ we define $\Delta^f(x, y_\ell, s_\ell) = \rbr{f^\star(x)^\top g^\star(y_\ell, s_\ell) - f(x)^\top g^\star(y_\ell, s_\ell)}^2$ and $X^f_\ell = \sup_{x \in \Xcal} \Delta^f(x, y_\ell, , s_\ell)$, which allows us to write $U_t = \sum_{\ell=1}^{t-1} X^{\fhat_t}_\ell$. 

We fix $f \in \Fcal$. We have $X^f_\ell \ge 0$ by definition and as $f^\star(x)^\top g^\star(y_\ell, s_\ell), f(x)^\top g^\star(y_\ell, s_\ell) \in [0, 1]$, we also have
\begin{equation}
    X^f_\ell = \sup_{x \in \Xcal} \Delta^f(x, y_\ell, s_\ell) = \sup_{x \in \Xcal}  \rbr{f^\star(x)^\top g^\star(y_\ell, s_\ell) - f(x)^\top g^\star(y_\ell, s_\ell)}^2 \le 1,
\end{equation}

Let $\PP_\ell \in \Delta(\Ycal)$ be the marginal distribution over $y_\ell$ conditioned on $\{x_{\ell'}, s_{\ell'}, y_{\ell'}, x_{\ell'}'\}_{\ell'=1}^{\ell-1} \cup \{  x_\ell, s_\ell\}$ then 
\begin{align*}
 \EE_{y \sim \PP_\ell} \sbr{ \left(X_\ell^f\right)^2 } &= \EE_{y \sim \PP_\ell}\sbr{\sup_{x \in \Xcal} \rbr{f^\star(x)^\top g^\star(y, s_\ell) - f(x)^\top g^\star(y, s_\ell)}^4} \\
 &\le 16 \EE_{y \sim \PP_\ell}\sbr{\TV{f^\star(\cdot)^\top g^\star(y, s_\ell) - f(\cdot)^\top g^\star(y, s_\ell)}^4},\\
 &\le 16 \EE_{y \sim \PP_\ell}\sbr{\TV{f^\star(\cdot)^\top g^\star(y, s_\ell) - f(\cdot)^\top g^\star(y, s_\ell)}^2},
\end{align*}
where we use the fact that $\|\cdot\|_\infty \le \|\cdot\|_1$, that TV distance between two distributions is equal to half of 1-norm distance between them, and that $\TV{\cdot} \le 1$. Thus, using the anytime Freedman's inequality (see~\pref{lem:freedman_anytime_simplified}) and union bound over all $f \in \mathcal{F}$, we get:
\begin{align*}
\sum_{\ell=1}^{t-1} X_\ell^f &\leq \eta \sum_{\ell=1}^{t-1}\EE_{y \sim \PP_\ell} \sbr{ \left(X_\ell^f\right)^2 } + \frac{C \log(t|\Fcal|/\delta)}{\eta} \\
&\leq 16\eta \sum_{\ell=1}^{t-1}\EE_{y \sim \PP_\ell}\sbr{\TV{f^\star(\cdot)^\top g^\star(y, s_\ell) - f(\cdot)^\top g^\star(y, s_\ell)}^2} + \frac{C \log(t|\Fcal|/\delta)}{\eta},
\end{align*}
simultaneously for all $t \in \NN$ and all $f \in \Fcal$ with probability at least $1-\delta$.

In particular by setting $f = \widehat f_t$ this implies,
\begin{equation}\label{eqn:bounding_I_byTV}
U_t = \sum_{\ell=1}^{t-1} X_\ell^{\widehat f_t} \leq 16\eta \sum_{\ell=1}^{t-1} \EE_{y \sim \PP_\ell}\sbr{\TV{f^\star(\cdot)^\top g^\star(y, s_\ell) - \widehat{f}_t(\cdot)^\top g^*(y, s_\ell)}^2} + \frac{C \log(t|\Fcal|/\delta)}{\eta}
\end{equation}
with probability at least $1-\delta$. Finally,~\pref{propn:mle} implies that with probability at least $1-\delta$, 
\begin{equation}\label{eqn:bounding_TV_high_prob_I}
\sum_{\ell=1}^{t-1} \EE_{y \sim \PP_\ell}\sbr{\TV{f^\star(\cdot)^\top g^\star(y, s_\ell) - \fhat_t(\cdot)^\top g^\star(y, s_\ell)}^2}  \leq  2\log(t|\Fcal|/\delta)
\end{equation}
for all $t \in \mathbb{N}$. Therefore a union bound implies that with probability at least $1-2\delta$, combining the bounds of~\pref{eqn:bounding_I_byTV} and~\pref{eqn:bounding_TV_high_prob_I},
\begin{equation*}
U_t = \sum_{\ell=1}^{t-1}X_\ell^{\widehat f_t} \leq 32\eta \log(t|\mathcal{F}|/\delta) + \frac{C\log(t|\Fcal|/\delta)}{\eta}
\end{equation*}
simultaneously for all $t \in \NN$. Optimizing for $\eta$ gives us $\eta = \sqrt{\frac{C}{32}}$, which gives us $U_t \le \sqrt{32C} \log(t|\Fcal|/\delta)$.

Plugging together the upper bounds for $U_t$ and $V_t$ we get:
\begin{equation*}
    \sup_{x \in \Xcal} \nbr{f^\star(x) - \fhat_t(x)}^2_{\widehat\Sigma_t} \le U_t + \lambda_t  V_t \le \sqrt{32C} \log(t|\Fcal|/\delta) + 4\lambda_t  B^2,
\end{equation*}
with probability $1-2\delta$ for all $t \in \NN$. This gives us the desired bound as $\sup_{x \in \Xcal} \nbr{f^\star(x) - \fhat_t(x)}_{\widehat\Sigma_t} \le \sqrt{\sqrt{32C} \log(t|\Fcal|/\delta) + 4\lambda_t  B^2} \le (32C)^{1/4} \sqrt{\log(t|\Fcal|/\delta)} + 2\sqrt{\lambda_t}B$ where we use $\sqrt{a + b} \le \sqrt{a} + \sqrt{b}$ for any $a, b \ge 0$.
\end{proof}

From now on we'll define as $\Ecal$ the event that \pref{eqn:upper_bound_cov_norm_single_point} holds for all $t \in \mathbb{N}$. As established in~\pref{lem:estimation_sigma_t_f}, we have $\PP(\Ecal) \ge 1 - 2\delta$. From~\pref{lem:estimation_sigma_t_f} we can also establish the following corollary.

\begin{corollary}\label{corr:differnce_dot_prod}
If event $\Ecal$ holds then we have:
\begin{equation*}
\abr{\rbr{f^\star(x) - \fhat_t(x)}^\top g^\star(y,s)} \le \rbr{(32C)^{1/4} \sqrt{\log(t|\Fcal|/\delta)} + 2\sqrt{\lambda_t}B}\nbr{g^\star(y,s)}_{\widehat \Sigma_t^{-1}}
\end{equation*}
for all $s \in \Scal$, $x \in \Xcal$ and $y \in \Ycal$ and all $t \in \NN$.
\end{corollary}
\begin{proof}
Let the event $\Ecal$ hold. Then the following sequence of inequalities holds,
\begin{align*}
\abr{\rbr{f^\star(x) - \fhat_t(x)}^\top g^\star(y,s)} 
&\stackrel{(i)}{\leq}  \nbr{f^\star(x)- \fhat_t(x )}_{\widehat\Sigma_t} \nbr{g^\star(y,s)}_{\widehat \Sigma_t^{-1}}, \\
&\stackrel{(ii)}{\leq}  \rbr{(32C)^{1/4} \sqrt{\log(t|\Fcal|/\delta)} + 2\sqrt{\lambda_t}B}\nbr{g^\star(y,s)}_{\widehat \Sigma_t^{-1}},
\end{align*}
where $(i)$ holds by Inequality~\ref{eqn:holder_ineq_probabilities_known_gstar} and $(ii)$ holds because of~\pref{lem:estimation_sigma_t_f} and because we are assuming $\Ecal$ holds.
\end{proof}

\subsection{Proving Optimism for $\algname$}

We next show that $\algname$ derives a useful optimism inequality that allows us to upper bound the true model probabilities $\PP(x \mid y,s) = f^\star(x)^\top g^\star(y,s)$ using the estimated model probabilities $\fhat_t(x)^\top g^\star(y,s)$.

\begin{lemma}\label{lem:optimism} If event $\Ecal$ holds, then we have:
    \begin{equation}
        f^\star(x)^\top g^\star(\pi^\star(x,s), s) \le \fhat_t(x)^\top g^\star(\pi_t(x,s),s) + \rbr{(32C)^{1/4} \sqrt{\log(t|\Fcal|/\delta)} + 2\sqrt{\lambda_t}B}\nbr{g^\star(\pi_t(x,s),s)}_{\widehat \Sigma_t^{-1}},
    \end{equation}
    for all $x \in \Xcal$ and $t \in \NN$, where $\pi_t$ is the policy in $t^{th}$ round of $\algname$ and $\pi^\star: x,s \mapsto \arg\max_{y \in \Ycal} \PP(x \mid y,s)$ is the optimal policy.
\end{lemma}
\begin{proof}
If event $\Ecal$ holds then for any $x \in \Scal, x \in \Xcal, y \in \Ycal$, and $t \in \NN$ we have directly from~\pref{corr:differnce_dot_prod} the following:
\begin{align*}
    f^\star(x)^\top g^\star(y,s) \le \fhat_t(x)^\top g^\star(y,s) + \rbr{(32C)^{1/4} \sqrt{\log(t|\Fcal|/\delta)} + 2\sqrt{\lambda_t}B}\nbr{g^\star(y,s)}_{\widehat \Sigma_t^{-1}}.
\end{align*}
For any $s \in \Scal$, $x \in \Xcal$ and $t \in \NN$, this implies:
\begin{align*}
    f^\star(x)^\top g^\star(\pi^\star(x,s),s) &\le \fhat_t(x)^\top g^\star(\pi^\star(x,s),s) + \rbr{(32C)^{1/4} \sqrt{\log(t|\Fcal|/\delta)} + 2\sqrt{\lambda_t}B}\nbr{g^\star(\pi^\star(x,s),s)}_{\widehat \Sigma_t^{-1}}\\
    &\le \fhat_t(x)^\top g^\star(\pi_t(x),s) + \rbr{(32C)^{1/4} \sqrt{\log(t|\Fcal|/\delta)} + 2\sqrt{\lambda_t}B}\nbr{g^\star(\pi_t(x),s)}_{\widehat \Sigma_t^{-1}},
\end{align*}
where the second inequality holds because of the definition of $\pi_t(x,s)$.
\end{proof}

\subsection{Regret Upper Bound}

We are ready to upper bound the regret of $\algname$(\pref{alg:known_gstar}). Recall that we define regret for a given run of $\algname$ as,
\begin{equation*}
  \Reg(T) = \sum_{t=1}^T \rbr{\PP(x_t \mid \pi^\star(x_t,s_t), s_t) - \PP(x_t \mid \pi_t(x_t,s_t), s_t)}.
\end{equation*}
The main result is the following theorem,

\regretboundalgorithm*

\begin{proof} Fix $T \in \NN$. We assume event $\Ecal$ holds, then we can bound the regret as:
\begin{align*}
    \Reg(T) &= \sum_{t=1}^T \rbr{\PP(x_t \mid \pi^\star(x_t,s_t), s_t) - \PP(x_t \mid y_t, s_t)},\\
        &=  \sum_{t=1}^T \rbr{f^\star(x_t)^\top g^\star(\pi^\star(x_t),s_t) - f^\star(x_t)^\top g^\star( \pi_t(x_t), s_t)},  \\
        &\stackrel{(i)}{\leq} \sum_{t=1}^T \left(\fhat_t(x_t)^\top g^\star(\pi_t(x_t),s_t) \right. \\
        & \qquad \quad + \left.\rbr{(32C)^{1/4} \sqrt{\log(t|\Fcal|/\delta)} + 2\sqrt{\lambda_t}B}\nbr{g^\star(\pi_t(x_t),s_t)}_{\widehat \Sigma_t^{-1}} - f^\star(x_t)^\top g^\star(\pi_t(x_t), s_t)\right),
\end{align*}
where we use~\pref{lem:optimism} in the last step. We also have:
\begin{equation*}
    \fhat_t(x)^\top g^\star(\pi_t(x_t),s_t) - f^\star(x_t)^\top g^\star(\pi_t(x_t), s_t) \le \rbr{(32C)^{1/4} \sqrt{\log(t|\Fcal|/\delta)} + 2\sqrt{\lambda_t}B}\nbr{g^\star(\pi_t(x_t), s_t)}_{\widehat \Sigma_t^{-1}}
\end{equation*}
due to~\pref{corr:differnce_dot_prod}. Combining these two results we get:
\begin{align}
    \Reg(T) &\le \sum_{t=1}^T 2\rbr{(32C)^{1/4} \sqrt{\log(t|\Fcal|/\delta)} + 2\sqrt{\lambda_t}B}\nbr{g^\star(\pi_t(x_t), s_t)}_{\widehat \Sigma_t^{-1}}\\
    &\le \sup_{t' \in [T]} \rbr{2\rbr{(32C)^{1/4} \sqrt{\log(t'|\Fcal|/\delta)} + 2\sqrt{\lambda_{t'}}B}} \cdot \sum_{t=1}^T \nbr{g^\star(\pi_t(x_t), s_t)}_{\widehat \Sigma_t^{-1}}
\end{align}

Let $\widetilde{\Sigma}_t = \sum_{\ell=1}^{t-1} g^\star(y_\ell, s_\ell)\left( g^\star(y_\ell, s_\ell)\right)^\top + \lambda_T \II  $ for all $t \in [T]$.  Further, let $D_t =  \sum_{\ell=1}^{t-1} g^\star(y_\ell, s_\ell)\left( g^\star(y_\ell, s_\ell)\right)^\top$, which gives us $\widetilde{\Sigma}_t = D_t + \lambda_T \II$ and $\widehat{\Sigma}_t = D_t + \lambda_t \II$. This implies that for any $v \in \RR^d$ we have
\begin{equation*}
    \nbr{v}_{\widetilde \Sigma_t} = \sqrt{v^\top (D_t + \lambda_T \II)v} = \sqrt{v^\top D_t v + \lambda_T \nbr{v}^2_2} \le \sqrt{v^\top D_t v + \lambda_t \nbr{v}^2_2} = \sqrt{v^\top (D_t + \lambda_t \II)v} = \nbr{v}_{\widehat \Sigma_t}.
\end{equation*}

As $\widehat{\Sigma}_t \succeq \widetilde{\Sigma}_t \succ 0$ and $\widehat{\Sigma}_t = \widetilde{\Sigma}_t + (\lambda_t - \lambda_T)\II$ with $\lambda_t - \lambda_T > 0$ it follows that by Proposition~\ref{prop:psd_order}, $\| v\|_{\widehat \Sigma_t^{-1} } \leq \| v\|_{\widetilde \Sigma_t^{-1}}$. This implies 
\begin{equation*}
\sum_{t=1}^T \| g^\star(\pi_t(x_t),s_t)\|_{\widehat \Sigma_t^{-1}} \leq \sum_{t=1}^T \| g^\star(\pi_t(x_t), s_t)\|_{\widetilde \Sigma_t^{-1}}
\end{equation*}
Finally,~\pref{prop:determinant_result} implies
 \begin{equation*} \sum_{t=1}^T \| g^\star(y_t,s_t)\|_{\widetilde \Sigma_t^{-1}} \leq  \sqrt{2T d \log\left(1 + \frac{TB^2}{\lambda_T}\right)  },
\end{equation*}
where we use $\sup_{y \in \Ycal, s \in \Scal}\nbr{g^\star(y,s)}_2 \le B$. Combining these we get:
\begin{equation*}
    \Reg(T) \le \sup_{t' \in [T]} \rbr{2\rbr{(32C)^{1/4} \sqrt{\log(t'|\Fcal|/\delta)} + 2\sqrt{\lambda_{t'}}B}} \sqrt{2T d \log\left(1 + \frac{TB^2}{\lambda_T}\right) }
\end{equation*}

Using $\lambda_t = 1/t \le 1$ we get:
\begin{align*}
    \Reg(T) &\le \rbr{2\rbr{(32C)^{1/4} \sqrt{\log(T|\Fcal|/\delta)} + 2 B}} \sqrt{2T d \log\left(1 + T^2B^2\right) }\\
    &\le 8\rbr{(2C)^{1/4} \sqrt{\log(T|\Fcal|/\delta)} + B} \sqrt{Td \log\left(1 + TB\right) },
\end{align*}
where we use $\log(1 + T^2B^2) \le 2 \log(1 + TB).$ This gives us 
\begin{equation*}
    \Reg(T) = \Ocal\rbr{B \sqrt{Td \log(1 + TB)} + \sqrt{Td \log(T|\Fcal|/\delta) \log(1 + TB)}}.
\end{equation*}

\end{proof}

\section{Useful Lemmas}\label{appendix::useful_lemmas}

\begin{lemma}[Hoeffding Inequality]\label{lemma::hoeffding} Let $\{ Y_\ell\}_{\ell=1}^\infty$ be a martingale difference sequence such that $Y_\ell$ is $Y_\ell \in [a_\ell, b_\ell]$ almost surely for some constants $a_\ell, b_\ell$ almost surely for all $\ell = 1, \cdots, t$. then 

\begin{equation*}
     \sum_{\ell=1}^t Y_\ell \leq  2\sqrt{\sum_{\ell=1}^t (b_\ell-a_\ell)^2 \ln\left(\frac{1}{\widetilde{\delta}}\right) }.
\end{equation*}
With probability at least $1-\widetilde{\delta}$.
\end{lemma}
See for example Corollary 2.20 from~\cite{wainwright2019high}.

Our results relies on the following variant of Bernstein inequality for martingales, or Freedman’s inequality \cite{freedman1975tail}, as stated in e.g., \cite{agarwal2014taming,beygelzimer2011contextual}.

\begin{lemma}[Simplified Freedman’s inequality]\label{lem:super_simplified_freedman}
Let $X_1, ..., X_T$ be a bounded martingale difference sequence with $|X_\ell| \le R$. For any $\delta' \in (0,1)$, and $\eta \in (0,1/R)$, with probability at least $1-\delta'$,
\begin{equation}
   \sum_{\ell=1}^{T} X_\ell \leq    \eta \sum_{\ell=1}^{T}  \EE_\ell[ X_\ell^2] + \frac{\log(1/\delta')}{\eta}.
\end{equation}
where $\EE_{\ell}[\cdot]$ is the conditional expectation\footnote{We will use this notation to denote conditional expectations throughout this work.} induced by conditioning on $X_1, \cdots, X_{\ell-1}$.
\end{lemma}

\begin{lemma}[Anytime Freedman]\label{lem:freedman_anytime_simplified}
Let $\{X_t\}_{t=1}^\infty$ be a bounded martingale difference sequence with $|X_t| \le R$ for all $t \in \mathbb{N}$. For any $\delta' \in (0,1)$, and $\eta \in (0,1/R)$, there exists a universal constant $C > 0$ such that for all $t \in \mathbb{N}$ simultaneously with probability at least $1-\delta'$,
\begin{equation}
   \sum_{\ell=1}^{t} X_\ell \leq    \eta \sum_{\ell=1}^{t}  \EE_\ell[ X_\ell^2] + \frac{C\log(t/\delta')}{\eta}.
\end{equation}
 where $\EE_{\ell}[\cdot]$ is the conditional expectation induced by conditioning on $X_1, \cdots, X_{\ell-1}$.\end{lemma}

\begin{proof}

This result follows from~\pref{lem:super_simplified_freedman}. Fix a time-index $t$ and define $\delta_t = \frac{\delta'}{12t^2}$.~\pref{lem:super_simplified_freedman} implies that with probability at least $1-\delta_t$,

\begin{equation*}
\sum_{\ell=1}^t X_\ell \leq \eta \sum_{\ell=1}^t \EE_\ell\left[ X_\ell^2 \right] + \frac{\log(1/\delta_t)}{\eta}.
\end{equation*}

A union bound implies that with probability at least $1-\sum_{\ell=1}^t \delta_t \geq 1-\delta'$,
\begin{align*}
\sum_{\ell=1}^t X_\ell &\leq \eta \sum_{\ell=1}^t \EE_\ell\left[ X_\ell^2 \right] + \frac{\log(12t^2/\delta')}{\eta}\\
&\stackrel{(i)}{\leq} \eta \sum_{\ell=1}^t \EE_\ell\left[ X_\ell^2 \right] + \frac{C\log(t/\delta')}{\eta}.
\end{align*}
holds for all $t \in \mathbb{N}$. Inequality $(i)$ holds because $\log(12t^2/\delta') = \mathcal{O}\left( \log(t\delta')\right)$.

\end{proof}

Adapted from Theorem~$21$ from~\cite{agarwal2020flambe}. See also~\cite{geer2000empirical}. 

\begin{proposition}[MLE Bound]\label{propn:mle} For any fixed $\delta \in (0, 1)$, 
\begin{equation*}
    \sum_{\ell=1}^{t-1} \EE_{y \sim \pi_h}\sbr{\TV{\fhat_t(\cdot)^\top g^\star(y,s_\ell) - f^\star(\cdot)^\top g^\star(y, s_\ell)}^2} \le 2\ln(t|\Fcal|/\delta)
\end{equation*}
for all $t \in \mathbb{N}$ simultaneously with probability at least $1-\delta$ where $\PP_\ell \in \Delta(\Ycal)$ is the distribution over $y_\ell$ conditioned on $\{x_{\ell'}, s_{\ell'}, y_{\ell'}, x_{\ell'}'\}_{\ell'=1}^{\ell-1} \cup \{  x_\ell, s_\ell\}$.   
\end{proposition}
\begin{proof}
This result is an immediate consequence of Theorem~$21$ from~\cite{agarwal2020flambe}. We convert this result into an anytime statement by invoking this result repeatedly with probability values $\delta_t = \frac{\delta}{3*t^2}$ and then applying a union bound. 
\end{proof}

\begin{proposition}[Proposition $3$ from~\cite{pmlr-v130-pacchiano21a}] \label{prop:determinant_result}
For any sequence of vectors $v_1, \cdots, v_t \subset \mathbb{R}^d$ satisfying $\| v_\ell \| \leq L$ for all $\ell \in \mathbb{N}$, let $\Sigma_t$ be its corresponding Gram matrix $\Sigma_t = \lambda \mathbb{I} + \sum_{\ell=1}^{t-1} v_\ell v_\ell^\top$. Then for all $t \in \mathbb{N}$, we have
\begin{equation*}
\sum_{\ell=1}^T \| v_\ell\|_{\Sigma_\ell^{-1}} \leq \sqrt{2T d \log\left(1 + \frac{TL^2}{\lambda}\right)  }
\end{equation*}
\end{proposition}

\begin{proposition}\label{prop:psd_order}
Let $A \succ 0$ be a $d\times d$ positive definite matrix. And let $\lambda' \geq 0$. If $v \in \mathbb{R}^d$, 
\begin{equation*}
\| v \|_{A + \lambda' \II} \geq \| v\|_A 
\end{equation*}
and
\begin{equation*}
\| v \|_{(A+\lambda' \II)^{-1}} \leq \| v\|_{A^{-1}} 
\end{equation*}

\end{proposition}

\begin{proof}
Let $v_1, \cdots, v_d$ be an orthonormal basis of eigenvectors of $A$. The eigenvalues of $A$ are positive because the matrix is assumed to be positive definite. Call $\mu_i >0 $ to the eigenvalue associated with eigenvector $v_i$. Elementary linear algebra shows that,
\begin{equation*}
(A+ \lambda' \II) v_i = Av_i + \lambda' v_i = (\mu_i + \lambda') v_i
\end{equation*}
thus showing that $v_i$ are also an orthonormal basis of eigenvectors for $A + \lambda' \II$ and have eigenvalues $\mu_i + \lambda'$. Thus,
\begin{align*}
    \| v\|_{A}^2=     v^\top \left( A \right) v 
    = \sum_{i=1}^d \mu_i \left(  \langle v, v_i \rangle\right)^2
     \leq \sum_{i=1}^d  (\mu_i+\lambda')\cdot \left(  \langle v, v_i \rangle\right)^2
    =  v^\top \left( A+ \lambda' \II) \right) 
    = \| v\|^2_{A+ \lambda' \II}
\end{align*}
Notice that $A^{-1} = \sum_{i=1}^d \frac{1}{\mu_i} v_i v_i^\top$ and $(A +\lambda' \II)^{-1} = \sum_{i=1}^d \frac{1}{\mu_i + \lambda'} v_i v_i^\top$. And therefore, 
\begin{align*}
    \| v\|_{A^{-1}}^2    =     v^\top  A^{-1}  v 
    = \sum_{i=1}^d  \frac{1}{\mu_i} \left(  \langle v, v_i \rangle\right)^2
    \geq \sum_{i=1}^d  \frac{1}{\mu_i+\lambda'} \left(  \langle v, v_i \rangle\right)^2
    =  v^\top \left( A+ \lambda' \II \right)^{-1} v
    = \| v\|_{(A+\lambda' \II)^{-1}}^2
\end{align*}
The result follows.
\end{proof}

\section{Experimental Details}

We provide additional details of our experiments in this section.

\subsection{Additional Details for the First Experiment on Synthetic Task}

We use the almost same implementation of $\algname$ as listed in~\pref{alg:known_gstar}. The only change we make is that we use a simplified policy given by:
\begin{equation}
    \pi_t(x_t) = \arg\max_{y \in \Ycal}\fhat_t(x_t)^\top g^\star(y, s_t) + k \nbr{g^\star(y, s_t)}_{\widehat \Sigma^{-1}_t},
\end{equation}
where $k$ is a single hyperparameter. We also use $\lambda_t = \lambda$ which is a hyperparameter to be tuned. We use the hyperparameter values in~\pref{tab:hyperparam_exp1}.

\begin{table}[!ht]
    \centering
    \begin{tabular}{c|c}
        \textbf{Hyperparameter} & \textbf{Values}\\
        \hline
        $\epsilon$ & grid search in $[0.05, 0.1, 0.2, 0.3]$  \\
        $\lambda$ &  grid search in $[0.05, 0.1, 1.0]$\\
        $k$ & grid search in $[0.1, 1.0, 10.0]$\\
        optimization & Adam\\
        learning rate & 0.001\\
        temperature used in defining $F$ and $G$ & 0.75\\
        $|\Xcal|$ & 2000\\
        $d$ & 10\\
        $|\Ycal|$ & 10\\
        \hline 
    \end{tabular}
    \caption{Grid search for hyperparameter for the first experiment.}
    \label{tab:hyperparam_exp1}
\end{table}

\subsection{Additional Details for the Second Experiment on the Image Selection Task}

\begin{table}[!ht]
    \centering
    \begin{tabular}{c|c}
        \textbf{Hyperparameter} & \textbf{Values}\\
        \hline
        $\epsilon$ & grid search in $[0.05, 0.1, 0.2, 0.3]$  \\
        $\lambda$ &  grid search in $[0.05, 0.1, 1.0]$\\
        $k$ & grid search in $[0.1, 1.0, 10.0]$\\
        optimization & Adam\\
        learning rate & 0.001\\
        vocabulary size & 34\\
        word embedding dimension & 10\\
        GRU hidden dimension & 10\\
        dimension of $g^\star$ encoding & 16\\
        number of layers in GRU & 2\\
        possible templates & 10\\
        object types & [``square", ``rectangle", ``triangle", ``circle"]\\
        object color & [``red", ``blue", ``green", ``yellow", ``black", ``grey", ``black", ``cyan", ``orange"]\\
        $|\Ycal|$ & 5\\
        \hline 
    \end{tabular}
    \caption{Grid search for hyperparameter for the second experiment.}
    \label{tab:hyperparam_exp2}
\end{table}

We use the hyperparameter values shown in~\pref{tab:hyperparam_exp2}. The list of templates is given below where \{object1\} and \{color1\} are variables that are replaced by the object type and its color in a given image.
\begin{enumerate}
    \item ``You are seeing a \{object1\} of color \{color1\}"
    \item ``The image contains a \{object1\} of color \{color1\}"
    \item ``There is a \{color1\} colored object of type \{object1\}"
    \item ``A \{color1\} \{object1\}"
    \item ``The object is a \{object1\} and its color is \{color1\}."
    \item  ``The image has a single \{color1\} colored \{object1\}."
    \item  ``You are seeing a \{color1\} colored \{object1\}."
    \item ``There is a \{color1\} colored object."     
    \item ``You are seeing a \{object1\} in the image."
    \item ``There is a \{color1\} colored {object1}."
\end{enumerate}

\paragraph{Autoencoder.} We use a 3-layer autoencoder with leaky relu activations. The first two layers apply a convolution with 16 $8\times 8$ kernels with stride 4. The last layer applies 15 $4 \times 4$ kernels with stride 2. The input image is an RGB image of size $200\times 200\times 3$. After applying the CNN encoder, we get a feature of size $16 \times 4 \times 4 = 64$. We flatten this feature and apply a fully connected layer to map it to another vector of size 64. We reshape it and pass it through a 3-layer deconvolutional network with leaky relu activation, to predict an image of the same size as the input. The first layer of the decoder applies 16 $4 \times 4$ convtranspose2d of stride 2 and output padding 1. The second layer applies 16 $8\times 8$ convtranspose2d of stride 4 and output padding 1. Finally, the last layer applies 16 $8\times 8$ convtranspose2d of stride 4 with no output padding. We train the autoencoder with squared loss using Adam optimization. We apply gradient clip to clip gradient above a clipping value of 2.5. Finally, we model $g^\star(y, s)$ by first applying the encoder to generate a feature map of $16 \times 4 \times 4$, and then summing over the first dimension and flattening the remaining tensor into a 16-dimensional vector.

\paragraph{Compute.} We use A2600 for all experiments. The entire set of experiments took 3 hours to finish. We used PyTorch to implement the code.

\end{document}